\pdfoutput=1
\documentclass{article}

\PassOptionsToPackage{authoryear,round}{natbib}
\usepackage[main,final]{neurips_2026}

\usepackage[utf8]{inputenc}
\usepackage[T1]{fontenc}
\usepackage{hyperref}
\hypersetup{hidelinks}
\usepackage{url}
\usepackage{booktabs}
\usepackage{amsmath,amsfonts,amssymb,amsthm}
\usepackage{graphicx}
\usepackage{multirow}
\usepackage{bm}
\usepackage{xcolor}
\usepackage{nicefrac}
\usepackage{microtype}
\usepackage{algorithm}
\usepackage{algpseudocode}
\usepackage{subcaption}
\usepackage[normalem]{ulem}
\usepackage{placeins}

\DeclareMathOperator*{\argmax}{arg\,max}
\DeclareMathOperator*{\argmin}{arg\,min}

\newtheorem{theo}{Theorem}
\newtheorem{lem}{Lemma}
\newtheorem{corollary}{Corollary}
\theoremstyle{definition}
\newtheorem{definition}{Definition}

\title{Ensemble Modeling for Time Series Forecasting: an Adaptive Robust Optimization Approach}

\author{%
L\'{e}onard Boussioux \\
Foster School of Business and Paul G. Allen\\
School of Computer Science \& Engineering\\
University of Washington, Seattle, WA 98195, USA\\
and Operations Research Center, MIT\\
\texttt{leobix@uw.edu}
\And
Henry Mao \\
Department of Information Systems\\
and Operations Management\\
University of Washington\\
Seattle, WA 98105, USA\\
\texttt{hmao25@uw.edu}
\AND
Dimitris Bertsimas \\
Sloan School of Management and Operations Research Center\\
Massachusetts Institute of Technology\\
Cambridge, MA 02139, USA\\
\texttt{dbertsim@mit.edu}
}

\begin{document}
\maketitle

\begin{abstract}
Accurate time-series forecasting is critical for many problems involving temporal data. Ensemble modeling is a well-established technique that leverages multiple predictive models to improve accuracy and robustness, since a single predictor's performance can vary widely due to shifts in the underlying data distribution. This paper proposes a new methodology for building robust ensembles of time series forecasting models. Our approach uses Adaptive Robust Optimization (ARO) to construct a linear regression ensemble whose model weights adapt over time. We demonstrate the effectiveness of our method through synthetic experiments covering eight disturbance scenarios, from no drift to abrupt regime shifts, and real-world applications, including air pollution management, energy consumption forecasting, and tropical cyclone intensity forecasting. In the real-world applications, our adaptive ensembles outperform the best ensemble member in hindsight by 16--26\% in root mean square error (RMSE) and 14--28\% in conditional value at risk at the 15\% level and improve over the other ensemble methods we compare on most metrics. In the synthetic experiments, they have lower mean RMSE than a rolling-ridge comparator in seven of the eight scenarios, significantly in four, while an abrupt regime shift that first appears after model selection remains a limitation.
\end{abstract}

\section{Introduction}\label{sec:Introduction}

Time series data capture sequential measurements and observations indexed by timestamps. In particular, predicting future target values through time-series forecasting is a critical research topic across many real-world applications, such as healthcare, finance, energy, climate, and business.

Many popular forecasting methods exist, including auto-regressive processes \citep{Hamilton+2020, hyndman}, exponential smoothing approaches \citep{gardner}, and neural networks \citep{salinas, pedrolara}. However, no forecasting method consistently outperforms others, as modeling assumptions may need to change over time to account for the often complex mechanisms underlying time series generation \citep{chatfield}. As a result, the accuracy of each forecasting method can vary significantly across time depending on model drift \citep{gama}. Throughout this paper, we use ``drift'' to mean changes in the error distribution or relative reliability of ensemble members over time; this is distinct from deterministic trend or seasonality in the target signal, which an individual forecaster may model directly.

Ensemble modeling is a well-established technique for leveraging the strengths and limitations of multiple models and benefiting from their diversity \citep{dietterich_ensemble_2000, brown, pmlr-v39-oliveira14}. When applied to time series forecasting, ensemble modeling can be complex, as the temporal aspect of the data makes it difficult to determine which model is most suitable at any given time \citep{pmlr-v39-oliveira14, cerqueira_evaluating_2020}. A common approach is to weight ensemble members based on their performance in recent or historical observations, assuming that the data distribution will remain similar and that models will behave consistently. Other approaches include meta-learning \citep{meta_prudencio, gastinger_study_2021} or regret minimization, for example, under the multi-armed bandit setting \citep{bianchi}.

However, these methods do not necessarily enforce the robustness of the ensembles, which can lead to critical prediction errors. In this study, we propose exploring ensemble modeling for time series forecasting from a different perspective: robust optimization. We present a novel optimization scheme that dynamically selects the weights of an ensemble of forecasts at each time step (Figure~\ref{fig:overview}), accounting for the uncertainty and errors of the underlying forecasts. Our approach is based on multi-stage robust optimization, where the weights are updated at each time step as new information about the uncertain parameters is revealed.

\begin{figure}[b]
\centering
\includegraphics[width=\linewidth]{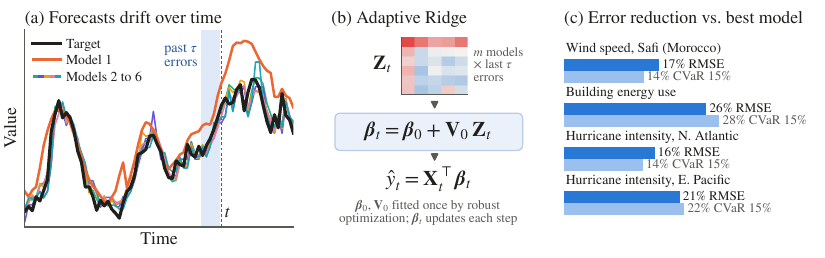}
\caption{Overview. (a) Six models forecast one target (data from the wind-speed case study, Section~\ref{sec:realworld}); Model 1 drifts in the second half. (b) Adaptive Ridge sets the weights $\boldsymbol{\beta}_t$ as an affine function of the models' last $\tau$ errors, with $\boldsymbol{\beta}_0$ and $\mathbf{V}_0$ fitted once by adaptive robust optimization. (c) Reduction in RMSE and CVaR 15\% relative to the best single model in each case study.}
\label{fig:overview}
\end{figure}

The Adaptive Robust Optimization (ARO) framework \citep{ben-tal_adjustable_2004} (background in Appendix~\ref{app:background}) is a methodology to address multi-stage problems by adjusting wait-and-see variables, decisions made after part of the uncertainty is revealed, as a function of uncertain parameters. By leveraging ARO, we aim to demonstrate the viability of creating a new branch of ensemble methods for temporal tasks, building on the strengths of adaptive robust methods. Our contributions are threefold: (i) an ARO-based adaptive ensemble, Adaptive Ridge, whose weights update from a sliding window of past forecast errors, with an extension to classification ensembles (Appendix~\ref{app:weighted-voting}); (ii) tractable equivalent regularized formulations with robustness guarantees; and (iii) synthetic evidence across eight disturbance scenarios, including an analysis of computational cost, sensitivity, and failure cases, and real-world evidence that the affine decision rule improves accuracy and tail-risk performance on the studied splits. Robustness here refers to perturbations of the member forecasts within an uncertainty set; a regime absent from the training history is outside this scope, and Section~\ref{sec:synthetic} studies it as a failure case. Our code is available at \url{https://github.com/leobix/adaptive_ensemble}.

\section{Adaptive robust ensembles for time series forecasting}\label{sec:methodology}

\subsection{Linear regression for a time series ensemble}

Define a sample $\mathbf{X}_t = [x_{t}^1, \dots, x_{t}^m] \in \mathbb{R}^m$ as the forecasts made by $m$ individual models at time $t$ for the same fixed lead time period. We call these $m$ models \emph{ensemble members}. We assume we are given a historical time series of ensemble members' predictions $(\mathbf{X},\mathbf{y}) \in (\mathbb{R}^{T\times m} \times \mathbb{R}^{T})$ with $T$ regularly spaced samples $\mathbf{X}_t \in \mathbb{R}^m$ with the corresponding ground-truth values $y_t \in \mathbb{R}$.

To combine the forecasts dynamically, we associate all ensemble member predictions with time-varying coefficients $\boldsymbol{\beta}_t = [\beta_t^{1}, \dots, \beta_t^{m}] \in \mathbb{R}^m$. For compactness, we also define the vector of coefficients
\[
\boldsymbol{\beta} := [\boldsymbol{\beta}_1| \dots| \boldsymbol{\beta}_T] \in \Omega \subset \mathbb{R}^{T \cdot m},
\]
where $\Omega$ is a user-designed set of constraints on the coefficients, and $|$ is the concatenation operation. To build an adaptive weighted ensemble that approximates the ground truth values, we consider problems similar to the ordinary least squares problem
\begin{equation}\label{eq:ols}
\min _{\boldsymbol{\beta} \in \Omega}\sqrt{\sum_{t=1}^{T}\left(y_{t}-\mathbf{X}^\top_{t} \boldsymbol{\beta}_t\right)^2},
\end{equation}
or the least absolute deviation problem $\min _{\boldsymbol{\beta} \in \Omega}\sum_{t=1}^{T}| y_{t}-\mathbf{X}^\top_{t} \boldsymbol{\beta}_t|$. For convenience and further developments, we rewrite these problems under a general compact form:
\begin{equation}\label{main_compact}
\min _{ \boldsymbol{\beta} \in \Omega} \left\|\mathbf{y}-\tilde{\mathbf{X}} \boldsymbol{\beta}\right\|,
\end{equation}
where $\left\|\cdot\right\|$ is a given norm, and $\tilde{\mathbf{X}}:=\left[\begin{array}{cccc}\mathbf{X}^\top_1 & \boldsymbol{0} & \cdots & \boldsymbol{0}\\ \boldsymbol{0} & \mathbf{X}^\top_2  & & \vdots \\ \vdots &  & \ddots & \boldsymbol{0} \\ \boldsymbol{0} & \cdots & \boldsymbol{0} & \mathbf{X}^\top_T \end{array}\right] \in \mathbb{R}^{T\times T \cdot m}$.

\subsection{Robustification and adaptive rule}\label{sec:adaptive_rule}

The nominal formulation \eqref{main_compact} is akin to the popular Sample Average Approximation (SAA) \citep{shapiro2003monte} minimizing the empirical error on the observed historical data. However, such formulation is prone to overfitting and poor out-of-sample performance \citep{smith2006optimizer}. Hence, a natural approach to avoid overfitting is to robustify the nominal formulation \eqref{main_compact} using an uncertainty set of possible perturbations $\boldsymbol{\Delta} \in \mathcal{U}$ of the forecast matrix $\mathbf{\tilde{X}}$. We consider, therefore, the robust formulation
\begin{equation}\label{RO1}
\min _{ \boldsymbol{\beta} \in \Omega}\max_{\boldsymbol{\Delta} \in \mathcal{U}}\left\|\mathbf{y}-(\tilde{\mathbf{X}}+\boldsymbol{\Delta}) \boldsymbol{\beta}(\boldsymbol{\Delta})\right\|,
\end{equation}
where $\mathcal{U}$ is an uncertainty set that characterizes the user's belief about perturbations of the historic forecasts $\mathbf{X}$ at each time step $t$. Here, the decision $\boldsymbol{\beta}$ becomes an adaptive variable $\boldsymbol{\beta}(\boldsymbol{\Delta})$ to the perturbation. Using a proper choice of uncertainty set $\mathcal{U}$, the formulation \eqref{RO1} accounts for ``adversarial noise'' in the forecast matrix $\mathbf{\tilde{X}}$ and seeks to protect the dynamic linear regression problem from structural uncertainty in the data.

We first decide on using a fixed window size of information $\tau \geq 1$ such that
\[
\boldsymbol{\beta}_t(\boldsymbol{\Delta}_1, \ldots, \boldsymbol{\Delta}_{t-1}) := \boldsymbol{\beta}_t(\boldsymbol{\Delta}_{t-\tau}, \ldots, \boldsymbol{\Delta}_{t-1}).
\]
This fixed window of past information makes the problem more tractable and practical by constraining the parameter size and ensuring that each $\boldsymbol{\beta}_t$ depends on the same number of historical time steps. We next consider that at time $t$, the uncertainties $(\boldsymbol{\Delta}_s)_{t>s\geq 1}$ have been observed as $(\widehat{\boldsymbol{\Delta}}_s)_{t>s\geq 1}$, and we use the proxy $\widehat{\boldsymbol{\Delta}}_s := \mathbf{X}_s - y_s \boldsymbol{e} \in \mathbb{R}^m$, the perturbation of the forecasts at time $s$, where $\boldsymbol{e} \in \mathbb{R}^m$ denotes the all-ones vector, i.e., we take the forecast uncertainties to be the previous forecast errors.

We now model $\boldsymbol{\beta}_t$ with an affine decision rule depending on the observed uncertainties:
\begin{equation}\label{eq:decisionrule}
\boldsymbol{\beta}_t:=\boldsymbol{\beta}_0 + \mathbf{V}_0 \cdot \mathbf{Z}_t,
\quad \mathbf{Z}_t := [\mathbf{X}_{t-\tau} -y_{t-\tau}\boldsymbol{e}| \ldots| \mathbf{X}_{t-1}-y_{t-1}\boldsymbol{e}] \in \mathbb{R}^{m\cdot \tau},
\end{equation}
where we define the variables $\boldsymbol{\beta}_0 \in \mathbb{R}^m$ and $\mathbf{V}_0 \in \mathbb{R}^{m\times m\cdot \tau}$. The matrix $\mathbf{V}_0$ is shared across all time steps: it maps the lagged forecast-error vector in the sliding window $\mathbf{Z}_t$ to a time-specific coefficient adjustment, rather than introducing a separate parameter matrix for each window. The parameters $\boldsymbol{\beta}_0$ and $\mathbf{V}_0$, which together define an error-to-weight map, are fitted once and then held fixed, whereas the weights $\boldsymbol{\beta}_t$ change at every time step with the observed errors. Overall, the coefficients adapt at each time step through this affine decision rule of the recent forecast errors.

The rest of the paper now focuses on this general and more compact formulation:
\begin{equation}\label{ARO2}
\min _{ \boldsymbol{\beta} \in \Omega_{\text{adapt}}}\max_{\boldsymbol{\Delta} \in \mathcal{U}}\left\|\mathbf{y}-(\tilde{\mathbf{X}}+\boldsymbol{\Delta}) \boldsymbol{\beta}\right\|,
\end{equation}
where $\Omega_{\text{adapt}}$ is defined as the set of all possible values for $\boldsymbol{\beta}$ that satisfy the affine decision rule:
\[
\begin{aligned}
\Omega_{\text{adapt}} & := \{\boldsymbol{\beta} \text{ subject to } \boldsymbol{\beta}_t =\boldsymbol{\beta}_0 + \mathbf{V}_0 \cdot \mathbf{Z}_t,  \boldsymbol{\beta}_0 \in \mathbb{R}^m,  \mathbf{V}_0 \in \mathbb{R}^{m\times m\cdot \tau}, \forall t \}, \\
\mathbf{Z}_t & := [\mathbf{X}_{t-\tau} -y_{t-\tau}\boldsymbol{e}| \ldots| \mathbf{X}_{t-1}-y_{t-1}\boldsymbol{e}] \in \mathbb{R}^{m\cdot \tau}.
\end{aligned}
\]
Once problem~\eqref{ARO2} is solved, with optimal parameters $\boldsymbol{\beta}_0^*$ and $\mathbf{V}_0^*$, predictions at time $T+k$ use the affine decision rule:
\[
\hat{y}_{T+k} = \mathbf{X}_{T+k}^\top \boldsymbol{\beta}_{T+k}=\mathbf{X}_{T+k}^\top (\boldsymbol{\beta}_0^*+\mathbf{V}_0^* \mathbf{Z}_{T+k}), \quad \forall k \geq 1.
\]

\subsection{Equivalence to a regularized problem}\label{sec:equivalence}

Even with the affine decision rule, it is a priori unclear how to conveniently solve the compact robust optimization problem \eqref{ARO2} due to its adversarial nature. Therefore, we now relate it to equivalent regularized regression problems for several relevant uncertainty sets. Let us consider a natural choice for the uncertainty set as
\[
\mathcal{U} = \{\boldsymbol{\Delta} \in \mathbb{R}^{T\times T\cdot m} : \left\| \boldsymbol{\Delta} \right\| \leq \lambda\},
\]
where $\left\| \cdot \right\|$ is some matrix norm or seminorm, and $\lambda > 0$.

\paragraph{Example 1: Adaptive Ridge.} The choice of the 2-Frobenius norm $\left\| \boldsymbol{\Delta} \right\| := \left\| \boldsymbol{\Delta} \right\|_{F_2}$ can be interpreted as a Euclidean ball centered at the observed data $\mathbf{X}$ with radius $\lambda$, which models a global perturbation in the data. In this case, using the $\ell_2$ norm and $\mathcal{U}_{F_2}:=\{\boldsymbol{\Delta}: \left\| \boldsymbol{\Delta} \right\|_{F_2} \leq \lambda\}$, the problem
\begin{equation}\label{ridge1}
\min _{ \boldsymbol{\beta} \in \Omega_{\text{adapt}}}\max_{\boldsymbol{\Delta} \in \mathcal{U}_{F_2}}\left\|\mathbf{y}-(\tilde{\mathbf{X}}+\boldsymbol{\Delta}) \boldsymbol{\beta}\right\|_2
\end{equation}
is equivalent to the following problem, which we name \emph{Adaptive Ridge}:
\begin{equation}\label{adaptive_ridge}
\min _{ \boldsymbol{\beta} \in \Omega_{\text{adapt}}}\left\|\mathbf{y}-\tilde{\mathbf{X}} \boldsymbol{\beta}\right\|_2+\lambda \left\|\boldsymbol{\beta}\right\|_2.
\end{equation}

\begin{definition}
For two norms $g,h$, we define the induced-norm uncertainty set
\[
\mathcal{U}_{(h,g)}:=\{\boldsymbol{\Delta} \in \mathbb{R}^{T \times T\cdot m} : \left\| \boldsymbol{\Delta} \right\|_{(h,g)} \leq \lambda\}, \quad \text{where} \quad \left\| \boldsymbol{\Delta} \right\|_{(h,g)} = \max_{\boldsymbol{\beta}\neq \boldsymbol{0}} \frac{g(\boldsymbol{\Delta} \boldsymbol{\beta})}{h(\boldsymbol{\beta})}.
\]
\end{definition}

\begin{definition}
For $p \in [1,+\infty]$, we define the Frobenius uncertainty set
\[
\mathcal{U}_{F_p}:=\{\boldsymbol{\Delta} \in \mathbb{R}^{T\times T \cdot m} : \left\| \boldsymbol{\Delta} \right\|_{F_p} \leq \lambda\}, \quad \text{where} \quad \left\| \boldsymbol{\Delta} \right\|_{F_p} = \Big(\sum_{i=1}^T\sum_{j=1}^{T\cdot m} {|\Delta_{ij}|^p}\Big)^{1/p}.
\]
\end{definition}

\begin{theo}\label{theo1}
For $p, q \in[1, \infty]$,
\[
\min _{\boldsymbol{\beta} \in \Omega_{\emph{adapt}}} \max _{\boldsymbol{\Delta} \in \mathcal{U}_{(q, p)}}\|\mathbf{y}-(\tilde{\mathbf{X}}+\boldsymbol{\Delta}) \boldsymbol{\beta}\|_{p}=\min _{\boldsymbol{\beta} \in \Omega_{\emph{adapt}}}\|\mathbf{y}-\tilde{\mathbf{X}} \boldsymbol{\beta}\|_{p}+\lambda\|\boldsymbol{\beta}\|_{q}.
\]
\end{theo}
In particular, for $p=q=2$, we recover what we call ``Adaptive Ridge'' in this paper as a robustification; likewise, for $p=2$ and $q=1$, we recover an $\ell_1$-regularized analogue.

\begin{proof}
The result follows from the lemma below.
\begin{lem}[\citealt{bertsimas2014}]\label{lemma4}
If $g: {\mathbb{R}^{m}} \rightarrow \mathbb{R}$ is a seminorm which is not identically zero and $h: \mathbb{R}^{n} \rightarrow \mathbb{R}$ is a norm, then for any $\mathbf{z} \in \mathbb{R}^{m}$ and $\boldsymbol{\beta} \in \mathbb{R}^{n}$,
\[
\max _{\boldsymbol{\Delta} \in \mathcal{U}_{(h, g)}} g(\mathbf{z}+\boldsymbol{\Delta} \boldsymbol{\beta})=g(\mathbf{z})+\lambda h(\boldsymbol{\beta}),
\]
where $\mathcal{U}_{(h, g)}=\left\{\boldsymbol{\Delta}:\|\boldsymbol{\Delta}\|_{(h, g)} \leq \lambda\right\}$.
\end{lem}
The triangle inequality directly gives $g(\mathbf{z}+\boldsymbol{\Delta} \boldsymbol{\beta}) \leq g(\mathbf{z})+\lambda h(\boldsymbol{\beta})$ for any $\boldsymbol{\Delta} \in \mathcal{U}:=\mathcal{U}_{(h, g)}$. We next show that there exists some $\boldsymbol{\Delta} \in \mathcal{U}$ so that $g(\mathbf{z}+\boldsymbol{\Delta} \boldsymbol{\beta})=g(\mathbf{z})+\lambda h(\boldsymbol{\beta})$. Let $\mathbf{v} \in \mathbb{R}^{n}$ so that $\mathbf{v} \in \operatorname{argmax}_{h^{*}(\mathbf{v})=1} \boldsymbol{\beta}^\top \textbf{v}$, where $h^{*}$ is the dual norm of $h$. Note in particular that $\boldsymbol{\beta}^\top \textbf{v}=h(\boldsymbol{\beta})$ by the definition of the dual norm $h^{*}$.

For now, suppose that $g(\mathbf{z}) \neq 0$. Define the rank one matrix $\widehat{\boldsymbol{\Delta}}:=\frac{\lambda}{g(\mathbf{z})} \mathbf{z} \mathbf{v}^{\top}$. We observe that
\[
g(\mathbf{z}+\widehat{\boldsymbol{\Delta}} \boldsymbol{\beta})=g\left(\mathbf{z}+\frac{\lambda h(\boldsymbol{\beta})}{g(\mathbf{z})} \mathbf{z}\right)=g(\mathbf{z})+\lambda h(\boldsymbol{\beta}).
\]
We next show that $\widehat{\boldsymbol{\Delta}} \in \mathcal{U}$. We remark that for any $\mathbf{x} \in \mathbb{R}^{n}$,
\[
g(\widehat{\mathbf{\Delta}} \mathbf{x})=g\left(\frac{\lambda \mathbf{v}^{\top} \mathbf{x}}{g(\mathbf{z})} \mathbf{z}\right)=\lambda\left|\mathbf{v}^{\top} \mathbf{x}\right| \leq \lambda h(\mathbf{x}) h^{*}(\mathbf{v})=\lambda h(\mathbf{x}),
\]
where the final inequality follows by definition of the dual norm. Hence $\widehat{\boldsymbol{\Delta}} \in \mathcal{U}$, as desired.

We now consider the case when $g(\mathbf{z})=0$. Let $\mathbf{u} \in {\mathbb{R}^{m}}$ so that $g(\mathbf{u})=1$. Let $\mathbf{v}$ be as before and define $\widehat{\boldsymbol{\Delta}}:=\lambda \mathbf{u v}^{\top}$. We observe that
\[
g(\mathbf{z}+\widehat{\boldsymbol{\Delta}} \boldsymbol{\beta})\leq g(\mathbf{z})+\lambda|\mathbf{v}^{\top}\boldsymbol{\beta}|g(\mathbf{u})=\lambda h(\boldsymbol{\beta}).
\]
Now, by the reverse triangle inequality, $g(\mathbf{z}+\widehat{\boldsymbol{\Delta}} \boldsymbol{\beta}) \geq g(\widehat{\boldsymbol{\Delta}} \boldsymbol{\beta})-g(\mathbf{z})=\lambda h(\boldsymbol{\beta})$, and therefore $g(\mathbf{z}+\widehat{\boldsymbol{\Delta}} \boldsymbol{\beta})=\lambda h(\boldsymbol{\beta})=g(\mathbf{z})+\lambda h(\boldsymbol{\beta})$. Theorem \ref{theo1} follows by applying Lemma \ref{lemma4} with $g=\|\cdot\|_p$ and $h=\|\cdot\|_q$ and passing the equality to the min.
\end{proof}

\begin{corollary}\label{theo2}
For any $p \in[1, \infty]$,
\[
\min _{\boldsymbol{\beta} \in \Omega_{\emph{adapt}}} \max _{\boldsymbol{\Delta} \in \mathcal{U}_{F_{p}}}\|\mathbf{y}-(\tilde{\mathbf{X}}+\boldsymbol{\Delta}) \boldsymbol{\beta}\|_{p}=\min _{\boldsymbol{\beta} \in \Omega_{\emph{adapt}}}\|\mathbf{y}-\tilde{\mathbf{X}} \boldsymbol{\beta}\|_{p}+\lambda\|\boldsymbol{\beta}\|_{p^{*}},
\]
where $1/p + 1/p^{*} = 1$ and $\mathcal{U}_{F_p}:=\{\boldsymbol{\Delta}  : \left\| \boldsymbol{\Delta} \right\|_{F_p} \leq \lambda\}$.
\end{corollary}

\paragraph{Proof of Corollary \ref{theo2}.}
We obtain the result as a corollary of the theorem below.
\begin{theo}[\citealt{xu_robust_2008}]
Let $p \geq 1$. Define $\mathcal{U}^{\prime} := \left\{ (\boldsymbol{\delta}_{1},\ldots, \boldsymbol{\delta}_{m}) | f_{j}(\|\boldsymbol{\delta}_{1}\|_{p},\ldots, \|\boldsymbol{\delta}_{m}\|_{p}) \leq 0; j=1, \ldots, k \right\}$ where $f_j(\cdot)$ are convex functions and $k, m \geq 1$. If the set ${\cal Z}:= \{{\bf z}\in \mathbb{R}^{m}\vert f_{j}({\bf z})\leq 0, j=1,\ldots, k; {\bf z}\geq {\bf 0}\}$ has a non-empty relative interior, then the robust regression problem
\[
\min _{{\boldsymbol{\beta}}\in \mathbb{R} ^{m}} \left \{\max _{\boldsymbol{\Delta} \in{\cal U}^{\prime}}\left\| \bf{y}-(\bf{X}+\boldsymbol{\Delta}) \boldsymbol{\beta}\right\|_p\right \}
\]
is equivalent to the following regularized regression problem:
\[
\min \limits_{\boldsymbol{\lambda}\in \mathbb{R}^{k}_{+},\boldsymbol{\kappa}\in \mathbb{R}^{m}_{+}, \boldsymbol{\beta}\in \mathbb{R}^{m}}\left\{\left\| {\bf y}-\bf{X} \boldsymbol{\beta}\right\|_p+v(\boldsymbol{\lambda}, \boldsymbol{\kappa}, \boldsymbol{\beta})\right\},
\]
where $v(\boldsymbol{\lambda}, \boldsymbol{\kappa}, \boldsymbol{\beta}):= \max_{{\bf c}\in \mathbb{R}^{m}}\left[(\boldsymbol{\kappa} +|\boldsymbol{\beta}|)^{\top} {\bf c}-\sum _{j=1}^{k}\boldsymbol{\lambda}_{j} f_{j}({\bf c})\right]$.
\end{theo}
By taking ${\cal U}^{\prime}={\cal U}_{F_p}=\left \{\left.(\boldsymbol{\delta }_{1},\ldots,\boldsymbol{\delta}_{{T\cdot m}})\right\vert \Vert \left \Vert \boldsymbol{\delta }_{1}\right\Vert_{p},\ldots,\Vert \boldsymbol{\delta }_{{T\cdot m}}\Vert _{p}\Vert _{p}\leq \lambda \right \}$ and constraining $\boldsymbol{\beta} \in \Omega_{\text{adapt}}$, we recover Corollary \ref{theo2}.

\subsection{Adaptive Ridge training and validation}\label{sec:pipeline}

Algorithm~\ref{algo:1} summarizes how Adaptive Ridge is trained, validated, and tested.

\begin{algorithm}[t]
\caption{Adaptive Ridge pipeline. $\Omega_{\text{adapt}}^{\tau}$ and $\mathbf{Z}_t^{\tau}$ denote $\Omega_{\text{adapt}}$ and $\mathbf{Z}_t$ of Section~\ref{sec:adaptive_rule} with window $\tau$; the subscripts tr, val, and te restrict $\mathbf{y}$ and $\tilde{\mathbf{X}}$ to a segment.}\label{algo:1}
{\raggedright
\begin{algorithmic}[1]
\Require Member forecasts $\mathbf{X}$ and targets $\mathbf{y}$ in time order; grids \mbox{$\Lambda = \{\lambda_1, \ldots, \lambda_L\}$} and \mbox{$\mathcal{T} = \{\tau_1, \ldots, \tau_J\}$}; segment sizes $n_{\text{tr}}, n_{\text{val}}$; validation metric $M$; test metrics $\mathcal{M}$
\Ensure $\boldsymbol{\beta}_0^*, \mathbf{V}_0^*$ and the test metrics
\State Split the time indices, without shuffling, into $S_{\text{tr}}$ (first $n_{\text{tr}}$), $S_{\text{val}}$ (next $n_{\text{val}}$), and $S_{\text{te}}$ (rest)
\For{$\lambda \in \Lambda$, $\tau \in \mathcal{T}$}
  \State $(\boldsymbol{\beta}_0^{\lambda,\tau}, \mathbf{V}_0^{\lambda,\tau}) \gets \argmin_{\boldsymbol{\beta} \in \Omega_{\text{adapt}}^{\tau}} \|\mathbf{y}_{\text{tr}}-\tilde{\mathbf{X}}_{\text{tr}} \boldsymbol{\beta}\|_2+\lambda \|\boldsymbol{\beta}\|_2$
  \State $\hat{y}_t \gets \mathbf{X}_t^\top(\boldsymbol{\beta}_0^{\lambda,\tau} + \mathbf{V}_0^{\lambda,\tau}\mathbf{Z}_t^{\tau})$ for $t \in S_{\text{val}}$
  \State $e(\lambda, \tau) \gets M\big((y_t)_{t \in S_{\text{val}}}, (\hat{y}_t)_{t \in S_{\text{val}}}\big)$
\EndFor
\State $(\lambda^*, \tau^*) \gets \argmin_{(\lambda,\tau) \in \Lambda \times \mathcal{T}} e(\lambda, \tau)$
\State $(\boldsymbol{\beta}_0^*, \mathbf{V}_0^*) \gets \argmin_{\boldsymbol{\beta} \in \Omega_{\text{adapt}}^{\tau^*}} \|\mathbf{y}_{\text{tr+val}}-\tilde{\mathbf{X}}_{\text{tr+val}} \boldsymbol{\beta}\|_2+\lambda^* \|\boldsymbol{\beta}\|_2$
\State $\hat{y}_t \gets \mathbf{X}_t^\top(\boldsymbol{\beta}_0^* + \mathbf{V}_0^*\mathbf{Z}_t^{\tau^*})$ for $t \in S_{\text{te}}$, where $\mathbf{Z}_t^{\tau^*}$ uses only errors observed before $t$
\State \Return $\boldsymbol{\beta}_0^*, \mathbf{V}_0^*$ and $\mathcal{M}\big((y_t)_{t \in S_{\text{te}}}, (\hat{y}_t)_{t \in S_{\text{te}}}\big)$
\end{algorithmic}
}
\end{algorithm}

The validation grid is $\Lambda \times \mathcal{T}$, and model selection uses only the validation segment. The test segment is used once, after selecting $(\lambda^*,\tau^*)$; in rolling validation and testing, $\mathbf{Z}_t$ is formed only from forecasts and targets with timestamps strictly earlier than $t$, so the target being predicted and all later test targets are never used to choose hyperparameters or produce $\hat{y}_t$. Across segment boundaries, the first predictions use the last $\tau$ observed errors of the preceding segment, without zero padding (Appendix~\ref{app:boundary}).

The affine decision rule has $d = m(1+\tau m)$ coefficients, so fitting for one window size and $L$ regularization values costs $O(Nd^2 + Ld^3)$ time on $N$ observations, and each prediction costs $O(\tau m^2)$. Appendix~\ref{app:implementation} gives the details.

\section{Experiments}\label{sec:experiments}

\subsection{Synthetic experiments}\label{sec:synthetic}

Our synthetic experiments aim to identify suitable conditions under which our adaptive ensemble method can outperform competitive methods and to provide guidance and intuition on the hyperparameters to use. They consist of two studies with 30 seeds each, run under the same protocol (Appendix~\ref{app:protocol}): sensitivity experiments that vary one problem dimension at a time under gradual drift (Figure~\ref{fig:synthetic_summary}), and a comparison of 15 methods across eight disturbance scenarios (Figure~\ref{fig:paired}). In both, we evaluated the different ensemble methods with Mean Absolute Error (MAE), Root Mean Square Error (RMSE), Mean Absolute Percentage Error (MAPE), Conditional Value at Risk 5\% (CVaR 5\%), and Conditional Value at Risk 15\% (CVaR 15\%). The main text reports mean test RMSE; Appendix~\ref{app:sweeps} gives all five metrics for the sensitivity experiments, and Tables~\ref{tab:sweeps_sd} and~\ref{tab:complete_sd} give the standard deviations of RMSE across seeds. We evaluated the Conditional Value at Risk $\operatorname{CVaR}_\alpha$ \citep{Rockafellar} that determines the expected loss once the $\alpha$ Value at Risk breakpoint has been breached:
\[
\operatorname{CVaR}_{\alpha}(\mathbf{y},\hat{\mathbf{y}}):=\min _{{\zeta}} {\zeta} + \frac{1}{\alpha T} \sum_{t=1}^{T}\max(0,|y_t-\hat{y}_t|-{\zeta}).
\]

\paragraph{Data generation.} All experiments use the same target process: a univariate time series of 4,000 time steps. We generated a periodic signal with some normally distributed noise:
\[
 y(t)=\sin\left(\frac{2 \pi}{500}\cdot t\right)+\epsilon(t), \quad \epsilon(t) \sim \mathcal{N}(0, 0.1).
\]
This deliberately simple periodic target isolates the ensemble-selection problem: the challenge comes from biased, noisy, and drifting member forecasts rather than from a complex target-generation process.
To simulate the availability of multiple forecasts from different predictive models at each time step, we randomly sampled the bias and standard deviation of the ensemble members' errors from specified ranges. Formally, at each time step $t$, we generated the values $\tilde{X}_k(t)$ of each ensemble member $k \in [m]$ as $\tilde{X}_k(t) = y(t)+\epsilon_k(t)$, with $\epsilon_k(t) \sim \mathcal{N}(b_k, \sigma_k)$, where we sampled $b_k \sim \operatorname{Uniform}(-0.5, 0.5)$ and $\sigma_k \sim \operatorname{Uniform}(0, 0.5)$ independently beforehand. To test the capacity of the ensembles to perform well when the reliability of their members changes over time, each final forecast adds a disturbance, $X_k(t) = \tilde{X}_k(t)+\delta_k(t)$. Our reference disturbance is gradual drift, $\delta_k(t) = \frac{t}{4000}\cdot\text{drift}_k(t)$, where $\text{drift}_k(t) \sim \mathcal{N}(b'_k, \sigma'_k)$, $b'_k \sim \mathcal{N}(0, \sigma_{\text{drift}})$ and $\sigma'_k \sim \operatorname{Uniform}(0, s_{\text{drift}})$. We also consider no disturbance, a sudden regime shift that begins inside the test segment after model selection, periodic drift, persistent bias, increasing noise in a random subset of members or in the initially best member, and strongly correlated member errors combined with gradual drift (Appendix~\ref{app:scenarios} gives the exact definitions). We split the data chronologically into training, validation, and test sets, and every method selects its hyperparameters on the validation segment only.

\begin{figure}[t]
\centering
\includegraphics[width=\linewidth]{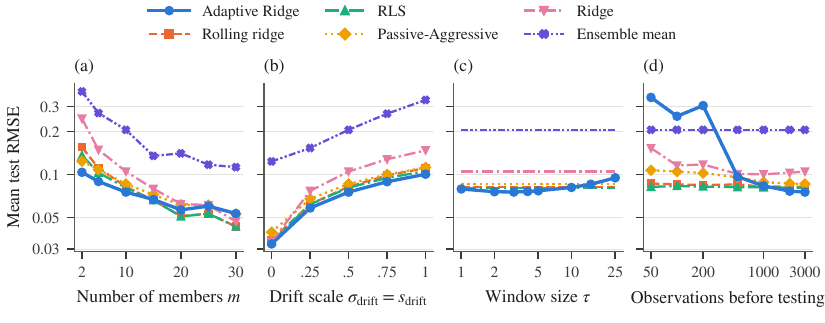}
\caption{Mean test RMSE over 30 seeds under gradual drift as we vary (a) the number of ensemble members, (b) the drift scale, (c) the window size $\tau$, with the regularization re-selected on validation for each $\tau$ (the other methods do not use $\tau$), and (d) the number of observations before the fixed 1,000-step test segment. Unless varied, $m = 10$, $\sigma_{\text{drift}} = s_{\text{drift}} = 0.5$, and 3,000 observations precede the test segment. The vertical axis is logarithmic. Appendix~\ref{app:sweeps} reports all eight methods and five metrics.}
\label{fig:synthetic_summary}
\end{figure}

\paragraph{Methods compared.} Along with the Adaptive Ridge method, we evaluated the best model in hindsight (an oracle static benchmark selected using test labels that is not deployable; it is included only to compare against the best single ensemble member), the ensemble mean, windowed exponentially weighted averaging (EWA) \citep{bianchi}, Passive-Aggressive \citep{passive}, and Ridge \citep{ridge} (Table~\ref{tab:ensemble_methods} in Appendix~\ref{app:protocol}). The sensitivity experiments add rolling ridge and recursive least squares (RLS), the two strongest online comparators. The comparison across disturbance scenarios adds eight methods from the online-learning and meta-learning literature: a validation-selected member, Hedge, Fixed Share, discounted ridge, two tree-based stackers (a regression tree and gradient-boosted trees), a softmax gate, and an arbitrated dynamic ensemble, for 15 deployable methods once the oracle is left out (Appendix~\ref{app:protocol}).

\paragraph{Sensitivity to the problem dimensions.} Under gradual drift, Figure~\ref{fig:synthetic_summary} varies the number of members, the drift scale, the window size, and the number of observations available for training and validation. Adaptive Ridge has a clear edge over ridge when sufficient data is available and when ensemble members may suffer from performance drift, and it has the lowest error at every drift scale (Figure~\ref{fig:synthetic_summary}b). Because the number of adaptive coefficients grows with $\tau m^2$ (Appendix~\ref{app:implementation}), a fixed history supports a limited number of members: with 3,000 observations before testing, Adaptive Ridge has the lowest error for up to 10 members, matches rolling ridge and RLS at 15, and falls behind them from 20 members onward (Figure~\ref{fig:synthetic_summary}a). Adaptive Ridge is not suitable when little data is available: it is unstable with 200 or fewer observations before testing, comparable to the best online methods at 1,000, and better at 2,000 (Figure~\ref{fig:synthetic_summary}d). With short histories, a purely online method such as RLS or rolling ridge should be preferred. Because of its additional variables, Adaptive Ridge needs $\lambda$ and $\tau$ to be validated and can overfit with long windows or little data; when the regularization is re-selected for each window, Adaptive Ridge stays below rolling ridge for windows of one to ten steps and degrades for longer windows (Figure~\ref{fig:synthetic_summary}c), which validation never selects (Appendix~\ref{app:sensitivity}).

\paragraph{Robustness across disturbance scenarios.} Figure~\ref{fig:paired} compares Adaptive Ridge with rolling ridge seed by seed under each scenario, with a common chronological split of 2,000 training, 1,000 validation, and 1,000 test observations. We pre-specified rolling ridge, which adapts by refitting on the most recent observations, as the primary comparator because it is the closest regularized linear alternative. Adaptive Ridge has lower mean RMSE than rolling ridge in seven of eight scenarios, significantly so in four after Holm correction, while rolling ridge is significantly better under the sudden shift. The advantage over rolling ridge also holds at every base-error correlation from 0 to 0.99 (Appendix~\ref{app:correlation}). Against all 15 methods (Rank in Figure~\ref{fig:paired}; Appendix~\ref{app:complete}), Adaptive Ridge has the lowest mean RMSE in six scenarios and the second lowest under persistent bias, after RLS; Passive-Aggressive (PA-I) is best under the sudden shift. Adaptive Ridge has lower mean RMSE than both tree-based stackers in every scenario, and than the softmax gate and the arbitrated ensemble in every scenario except the sudden shift. The sudden shift is the main failure case: because the error-to-weight map is fitted before testing, a shift that first appears after model selection produces error patterns that are absent from the training data, whereas online methods update after each new outcome.

\begin{figure}[t]
\centering
\includegraphics[width=\linewidth]{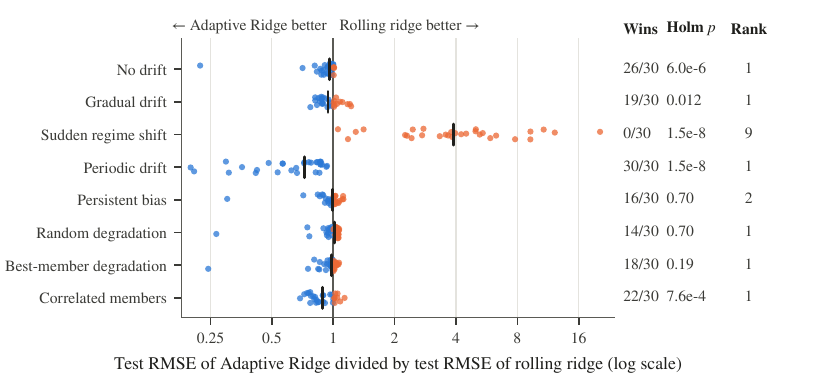}
\caption{Ratio of the test RMSE of Adaptive Ridge to that of rolling ridge for each of 30 paired seeds (dots) and their median (bar), by disturbance scenario; values below 1 favor Adaptive Ridge. Wins count the seeds with a ratio below 1, and the $p$-values come from two-sided paired Wilcoxon tests, Holm-corrected across the eight scenarios. Rank is the position of Adaptive Ridge among the 15 methods of Appendix~\ref{app:complete} by mean RMSE; unlike the win count, it compares averages over seeds.}
\label{fig:paired}
\end{figure}

\subsection{Real-world case studies}\label{sec:realworld}

This section illustrates the benefits of our adaptive ensemble method with real-world data. We investigate three applications in which the time series characteristics differ: air pollution management through wind speed forecasting, energy use forecasting, and tropical cyclone intensity forecasting. In these case studies, we compare the core methods of Table~\ref{tab:ensemble_methods} (grids in Appendix~\ref{app:realworld_hp}) and, for tropical cyclones, two operational forecasts.

\paragraph{Next-hour wind speed forecasting for pollution management.} Chemical plants can limit the pollution that reaches nearby cities by adapting their operations to the weather. With reliable wind forecasts, a plant can plan ahead and reduce or pause production before the wind carries toxic emissions toward the city. We have access to the 8,494 historical one-hour lead time wind speed forecasts made by 6 models \citep{safi} during 2020 for every hour: the Moroccan weather agency operational forecast, XGBoost, Decision Trees, Optimal Regression Trees, Lasso Regression, and Ridge Regression. The ground truth was measured by a sensor on-site that collected data every minute and then was averaged hourly on the industrial site in Safi, Morocco. Appendix~\ref{app:realworld_members} lists the members, and Appendix~\ref{app:safi_coefs} shows how their fitted weights evolve hour by hour.

Adaptive Ridge consistently provides the best performance across all metrics (Table~\ref{tab:safi}), improving over the best model in hindsight by 8\% in MAE, 17\% in RMSE, 7\% in MAPE, 26\% in CVaR 5\%, and 14\% in CVaR 15\%. It also improves over ridge on every metric, most on the largest errors (17\% in CVaR 5\%), which matter most to plant operators preparing for dangerous weather.

\begin{table}[t]
\caption{Results for each metric and ensemble method on the test set for the wind speed forecasting task on the Safi site. Results are given in m/s.}
\centering
\small
\begin{tabular}{lccccc}
\toprule
\textbf{Ensemble Method} & \textbf{MAE} & \textbf{RMSE} & \textbf{MAPE (\%)} & \textbf{CVaR 5\%} & \textbf{CVaR 15\%} \\
\midrule
Best Model in Hindsight (oracle) & $0.507$ & $0.756$ & $15.9$ & $2.27$ & $1.46$ \\
Ensemble Mean & $0.553$ & $0.783$ & $17.4$ & $2.27$ & $1.52$ \\
Windowed EWA & $0.510$ & $0.690$ & $16.5$ & $1.87$ & $1.36$ \\
Passive-Aggressive & $0.643$ & $0.868$ & $20.1$ & $2.35$ & $1.72$ \\
Ridge & $0.521$ & $0.718$ & $16.7$ & $2.03$ & $1.42$ \\
Adaptive Ridge & $\mathbf{0.469}$ & $\mathbf{0.626}$ & $\mathbf{14.8}$ & $\mathbf{1.68}$ & $\mathbf{1.25}$ \\
\bottomrule
\end{tabular}\label{tab:safi}
\end{table}

\paragraph{Energy consumption forecasting.} Energy demand in buildings has grown quickly over the past decades. We compare the ensemble methods on appliance energy consumption data provided by \cite{candanedo}. The dataset is publicly available at the University of California Irvine (UCI) Machine Learning repository. We trained several forecasting models that constitute our ensemble members and selected the best 10 with respect to MAE on the members' own validation set: Ordinary Least Squares (OLS), Ridge, Bayesian Ridge, ElasticNet, Huber Regressor, Least Angle Regression (LARS), Lasso, LassoLARS, Linear Support Vector Regression, and Orthogonal Matching Pursuit. These ensemble-member identities are summarized in Appendix \ref{app:realworld_members}.

The results presented in Table \ref{tab:energy} demonstrate the superiority of the Adaptive Ridge method compared to other ensemble techniques. It is the only ensemble method that outperformed the best model in hindsight in both MAE and RMSE, showing a 15\% improvement in MAE and a 26\% improvement in RMSE. The energy consumption target values in this dataset displayed significant bursts, making it difficult for all ensemble methods to accurately capture such events, as evidenced by the poor CVaR scores. Adaptive Ridge still has the lowest CVaR of all methods, 27\% lower in CVaR 5\% and 28\% lower in CVaR 15\% than the best ensemble member.

\begin{table}[t]
\caption{Results for each metric and ensemble method on the test set for the energy consumption task. Results are given in Wh.}
\centering
\small
\begin{tabular}{lccccc}
\toprule
\textbf{Ensemble Method} & \textbf{MAE} & \textbf{RMSE} & \textbf{MAPE (\%)} & \textbf{CVaR 5\%} & \textbf{CVaR 15\%} \\
\midrule
Best Model in Hindsight (oracle) & $31.1$ & $80.0$ & $\mathbf{23.2}$ & $309.0$ & $158.3$ \\
Ensemble Mean & $33.3$ & $79.4$ & $27.6$ & $304.3$ & $157.6$ \\
Windowed EWA & $32.9$ & $78.8$ & $27.1$ & $301.9$ & $156.4$ \\
Passive-Aggressive & $52.2$ & $149.1$ & $45.0$ & $476.7$ & $249.8$ \\
Ridge & $35.3$ & $76.4$ & $33.7$ & $287.5$ & $145.0$ \\
Adaptive Ridge & $\mathbf{26.3}$ & $\mathbf{59.4}$ & $23.5$ & $\mathbf{224.9}$ & $\mathbf{113.8}$ \\
\bottomrule
\end{tabular}\label{tab:energy}
\end{table}

\paragraph{Tropical cyclone intensity forecasting.} Tropical cyclones are a considerable threat to communities, causing hundreds of deaths and billions of dollars in damage each year. Intensity, i.e., the maximum sustained wind speed of a storm over a 1-minute interval, is one of the most important and difficult characteristics to forecast and directly impacts decision-making. In this real-world use case, we evaluate the effectiveness of the different ensemble methods for 24-hour lead time tropical cyclone intensity forecasting, which is critical to undertake life-saving measures and mitigate the impact of these devastating natural disasters. For a 24-hour forecast issued at time $t$, the adaptive weights use only forecasts and intensities verified by time $t$ (Appendix~\ref{app:boundary}).

Adaptive Ridge and Ridge have lower MAE, RMSE, and CVaR 15\% than all other ensembles in both basins (Table~\ref{tab:hurricanes}), including the current operational models OFCL and FSSE, which suggests they could help improve the official forecasts of the National Hurricane Center. Lower CVaR matters here because the largest errors drive poor decisions on the ground. Adaptive Ridge has the best performance overall across basins and metrics, with only ridge slightly surpassing it in CVaR on the North Atlantic basin. The adaptive component provides a valuable performance boost compared to ridge only, such as an MAE gain of 4\% on the North Atlantic basin and 9\% on the Eastern Pacific basin. The hurricane ensemble members, which include the Hurricast models of \citet{boussioux}, and the operational benchmarks are summarized in Appendix \ref{app:hurricane_members}.

\begin{table}[h]
\centering
\caption{Results in knots for each metric and ensemble method for 24-hour lead time intensity forecasting task. Bold values highlight the best performance for each metric.}
\small
\resizebox{\textwidth}{!}{\begin{tabular}{lcccccccccc}
\toprule
& \multicolumn{5}{c}{North Atlantic Basin} & \multicolumn{5}{c}{Eastern Pacific Basin} \\
\cmidrule(lr){2-6}\cmidrule(lr){7-11}
\textbf{Ensemble Method} & \textbf{MAE} & \textbf{RMSE} & \textbf{MAPE (\%)} & \textbf{CVaR 5\%} & \textbf{CVaR 15\%} & \textbf{MAE} & \textbf{RMSE} & \textbf{MAPE (\%)} & \textbf{CVaR 5\%} & \textbf{CVaR 15\%} \\
\midrule
Best Model in Hindsight (oracle) & 9.2 & 12.2 & 14.2 & 31.4 & 24.6 & 11.3 & 15.3 & 14.5 & 41.1 & 30.8 \\
Ensemble Mean & 9.4 & 12.8 & 14.0 & 35.2 & 25.6 & 10.9 & 14.6 & 14.2 & 38.2 & 29.2 \\
Windowed EWA & 8.8 & 11.5 & 13.3 & 30.0 & 22.6 & 10.3 & 14.9 & 15.5 & 43.3 & 31.1 \\
Passive-Aggressive & 9.1 & 12.4 & 13.9 & 34.6 & 25.3 & 17.9 & 22.9 & 27.1 & 57.2 & 44.4 \\
FSSE & 8.8 & 11.6 & 13.8 & 31.0 & 23.2 & 10.0 & 14.4 & 15.5 & 41.3 & 28.9 \\
OFCL & 8.4 & 11.4 & 13.3 & 29.8 & 22.2 & 10.6 & 15.6 & 17.2 & 44.3 & 30.6 \\
Ridge & 7.7 & 10.3 & 11.7 & \textbf{27.2} & \textbf{21.0} & 9.7 & 13.6 & 15.0 & 39.2 & 27.8 \\
Adaptive Ridge & \textbf{7.4} & \textbf{10.2} & \textbf{11.6} & 28.6 & 21.1 & \textbf{8.8} & \textbf{12.1} & \textbf{13.4} & \textbf{34.1} & \textbf{24.1} \\
\bottomrule
\end{tabular}}
\label{tab:hurricanes}
\end{table}

\section{Conclusion}

Using an ARO approach, this paper presented a novel methodology for building robust ensembles of time series forecasting models. This approach treats the weights of the ensemble members as wait-and-see decisions, which an affine decision rule adjusts to their latest errors, and transforms the resulting min-max problem into a tractable form with equivalence theorems.

We have demonstrated the effectiveness of our adaptive ensemble mechanism through a series of synthetic and real-world experiments. Across eight synthetic disturbance scenarios, it had lower mean RMSE than rolling ridge in seven, significantly in four, and the lowest among 15 methods, including online-learning and meta-learning combiners, in six. We have also applied our technique to three real-world use cases -- weather forecasting for air pollution management, energy consumption forecasting, and tropical cyclone intensity forecasting -- where it reduced RMSE by 16--26\% and CVaR 15\% by 14--28\% relative to the best ensemble member.

One of the advantages of our method is that it relies on a robust linear formulation, in which each weight is an explicit function of recent errors; this makes it easy to deploy and to interpret, and with the closed-form implementation the complete pipeline takes about 0.1 seconds for 10 members and 3,000 observations (Appendix~\ref{app:runtime}). Our results also indicate where the method is most valuable: in tasks where robustness and risk minimization are critical, and where the reliability of the ensemble members changes gradually or periodically rather than abruptly.

\paragraph{Limitations.} The method assumes a chronological record of ensemble-member forecasts and verified outcomes, and the affine decision rule can overfit when the training set is small, the sliding window is too large, or the ensemble has many members relative to the available history. In our synthetic experiments with ten members, the method needed on the order of a thousand observations for training and validation to match the best online combiners, which are preferable with shorter histories. Because its error-to-weight map is held fixed during testing, the method is also not designed for immediate adaptation to a previously unseen regime (Section~\ref{sec:synthetic}). Finally, the case studies compare only the core methods of Table~\ref{tab:ensemble_methods}; the additional methods of Section~\ref{sec:synthetic} were not run on these data. 

Beyond the applications studied here, adaptive ensembles could support forecasting in epidemiology, healthcare operations, supply chains, and product demand, and the same affine rule could select features adaptively in multimodal models. Future work could extend the framework to probabilistic forecasts, evaluate the classification extension of Appendix~\ref{app:weighted-voting} on data, combine the fitted rule with online updates for abrupt shifts, and study automated member selection and richer adaptive rules with explicit complexity control.

\begin{ack}
We thank Shuvomoy Das Gupta, Amine Bennouna, and Mo\"{\i}se Blanchard for useful discussions. The authors acknowledge the MIT SuperCloud and Lincoln Laboratory Supercomputing Center for providing high-performance computing resources that have contributed to the research results reported within this paper. This work received no third-party funding. The authors declare no competing interests.
\end{ack}

\bibliographystyle{plainnat}
{\small
\bibliography{References}
}

\clearpage
\appendix
\section{Additional methodological details}\label{app:method}

\subsection{Background on ensembles and adaptive robust optimization}\label{app:background}

{
\paragraph{Linear and dynamic ensembles.} After the seminal work of \citet{bates}, several combination methods were added to the forecaster's toolbox \citep{clemen_combining_1989, zou_combining_2004}. Weighted linear combinations of ensemble members became popular because they are simple to deploy. Typical combinations include the simple, trimmed, or winsorized average and the median of the members' forecasts; other methods weight the members by their past errors, performance, or variance, for instance through a regularized regression such as ridge \citep{ridge} or the lasso \citep{lasso}; see \citet{litreview} for a review. These weights are usually trained on historical data and then held fixed, which becomes a limitation under model drift, when the members' accuracy changes over time because the distribution of the inputs or of the target changes \citep{gama}. Dynamic ensembles instead update the weights over time \citep{kuncheva, gama}, often with regret guarantees relative to the best member, as in prediction with expert advice and multi-armed bandits \citep{bianchi, lattimore}.

\paragraph{Robust optimization.} Robust optimization (RO) protects a decision against adversarial perturbations of the data by introducing an uncertainty set that describes these perturbations \citep{bertsimas_robust}. Let $\mathbf{x} \in \mathcal{X}$ be a decision, $\mathbf{z}$ the uncertain parameters, $f(\mathbf{z},\mathbf{x})$ the objective, and $f_i(\mathbf{z},\mathbf{x}) \leq 0$, $i \in [p]$, the constraints. When $\mathbf{z}$ is only known to lie in an uncertainty set $\mathcal{Z}$, RO optimizes the worst case:
\[
\min_{\mathbf{x} \in \mathcal{X}} \max_{\mathbf{z} \in \mathcal{Z}} f(\mathbf{z}, \mathbf{x}) \quad \text{s.t.} \quad f_i(\mathbf{z}, \mathbf{x}) \leq 0, \quad \forall i \in [p],\ \forall \mathbf{z} \in \mathcal{Z}.
\]
Such min-max problems can be hard to solve directly, so a central step in RO is to derive an equivalent formulation that is computationally tractable, as Section~\ref{sec:equivalence} does for our ensemble problem.

\paragraph{Adaptive robust optimization.} Adaptive robust optimization (ARO) splits the decisions into stages: ``here-and-now'' decisions are made before the uncertainty is revealed, and ``wait-and-see'' decisions $\mathbf{y}(\mathbf{z})$ are made after part of it is observed \citep{ben-tal_adjustable_2004}. This setting covers many multi-stage problems, such as inventory control, lot sizing, and unit commitment \citep{unitcommitment, lot}. An ARO problem reads
\[
\min_{\mathbf{x} \in \mathcal{X},\, \mathbf{y}(\cdot)} \max_{\mathbf{z} \in \mathcal{Z}} f(\mathbf{z}, \mathbf{x}, \mathbf{y}(\mathbf{z})) \quad \text{s.t.} \quad f_i(\mathbf{z}, \mathbf{x}, \mathbf{y}(\mathbf{z})) \leq 0, \quad \forall i \in [p],\ \forall \mathbf{z} \in \mathcal{Z},
\]
where the function $\mathbf{y}(\cdot)$, called a decision rule, is chosen before $\mathbf{z}$ is known. Problems with general decision rules are NP-hard \citep{ben-tal_adjustable_2004}, so $\mathbf{y}(\cdot)$ is often restricted to the affine decision rule $\mathbf{y}(\mathbf{z}) = \mathbf{x}_0 + \mathbf{V}\mathbf{z}$, whose coefficients $\mathbf{x}_0$ and $\mathbf{V}$ become here-and-now decisions. Our ensemble weights follow exactly this structure: $\boldsymbol{\beta}_t = \boldsymbol{\beta}_0 + \mathbf{V}_0\mathbf{Z}_t$ in \eqref{eq:decisionrule}, with the recent forecast errors $\mathbf{Z}_t$ playing the role of the revealed uncertainty.
}

\subsection{Computational complexity}\label{app:implementation}

Let $N$ be the number of fitting observations for which a full error history is available, $m$ the number of ensemble members, $\tau$ the error-history window, and $d = m(1+\tau m) = O(\tau m^2)$ the number of coefficients of the affine rule~\eqref{eq:decisionrule}. For each window size $\tau$, the closed-form implementation used for the synthetic experiments of Section~\ref{sec:synthetic} builds the lagged regression design and the required cross-products once, in $O(Nd^2)$ time and $O(Nd + d^2)$ peak memory. These matrices are reused across the $L$ regularization candidates, each of which requires an $O(d^3)$ linear-system solve. The cost for one window size is therefore $O(Nd^2 + Ld^3)$, and the complete validation search sums this cost over the evaluated window sizes. Each prediction costs $O(\tau m^2)$. Because $d$ grows linearly in $\tau$ and quadratically in $m$, the dense formulation can become expensive for large ensembles or long windows. Appendix~\ref{app:runtime} reports measured runtimes for Adaptive Ridge and nine baselines.

In practice, both implementations minimize the squared form of \eqref{adaptive_ridge} plus a ridge penalty on the parameters of the rule, $\frac{1}{N}\|\mathbf{y}-\tilde{\mathbf{X}}\boldsymbol{\beta}\|_2^2 + \rho\big(\frac{1}{N}\sum_{t}\|\boldsymbol{\beta}_t\|_2^2 + \|\boldsymbol{\beta}_0\|_2^2 + \|\mathbf{V}_0\|_F^2\big)$. Without the last two terms, the squared and unsquared problems have the same solutions for corresponding values of $\rho$ and $\lambda$; the added terms keep the rule stable when the error histories at test time differ from those seen in training. Both implementations also treat a constant forecast as an additional member, which gives the ensemble an adaptive intercept (Figure~\ref{fig:evolution}), and append constant entries to the error history $\mathbf{Z}_t$.

\subsection{Initialization and boundary handling}\label{app:boundary}

Training, validation, and test sets are consecutive portions of the same time series. To predict the first validation observation, we construct its error history from the final $\tau$ training observations: their targets are already known, so their member forecast errors are available. After selecting the hyperparameters and refitting on training plus validation, the first test prediction analogously uses the final $\tau$ observed errors. Every later prediction uses only errors whose outcomes have already been observed; we never use the target being predicted, any future target, or zero padding. For forecasts verified with a delay, such as the 24-hour tropical cyclone forecasts, an error enters the history only once its verifying observation is available. The reported evaluations therefore do not encounter a fully cold start. If deployment begins at the first observation of a new series without previous outcomes, the adaptive adjustment cannot yet be constructed; a practical convention is to use the static component $\boldsymbol{\beta}_0$ for the first $\tau$ predictions and then switch to the adaptive rule. We did not evaluate this convention.

\subsection{Extension to robust adaptive weighted voting}\label{app:weighted-voting}

The affine rule can also combine classifiers. Let $\mathbf{S}_t \in \mathbb{R}^{K \times m}$ contain the scores assigned to $K$ classes by $m$ base classifiers at time $t$; one-hot score columns include hard votes as a special case. Given a vector $\mathbf{z}_t$ of recent classification feedback that is observable before time $t$, such as lagged margins or classification errors, define
\[
\boldsymbol{\beta}_t = \boldsymbol{\beta}_0 + \mathbf{V}_0 \mathbf{z}_t, \qquad \mathbf{g}_t = \mathbf{S}_t \boldsymbol{\beta}_t, \qquad \hat{y}_t \in \argmax_{k \in [K]} g_{t,k}.
\]

\begin{lem}[Robust multiclass hinge counterpart]\label{lem:robust-weighted-voting}
Let $y_t \in [K]$, let $\mathbf{s}_{t,k}^\top$ denote row $k$ of $\mathbf{S}_t$, and consider the multiclass hinge loss
\[
\ell_t(\mathbf{S}_t, \boldsymbol{\beta}_t) = \max\Big\{0, \max_{k \neq y_t} \big[1 + (\mathbf{s}_{t,k} - \mathbf{s}_{t,y_t})^\top \boldsymbol{\beta}_t\big]\Big\}.
\]
If the score matrix is perturbed within $\mathcal{U}_t = \{\boldsymbol{\Delta}_t : \|\boldsymbol{\Delta}_t\|_F \leq r_t\}$, then
\[
\max_{\boldsymbol{\Delta}_t \in \mathcal{U}_t} \ell_t(\mathbf{S}_t + \boldsymbol{\Delta}_t, \boldsymbol{\beta}_t) = \max\Big\{0, \max_{k \neq y_t} \big[1 + (\mathbf{s}_{t,k} - \mathbf{s}_{t,y_t})^\top \boldsymbol{\beta}_t + \sqrt{2}\, r_t \|\boldsymbol{\beta}_t\|_2\big]\Big\}.
\]
Consequently, minimizing the sum of worst-case losses over $(\boldsymbol{\beta}_0, \mathbf{V}_0)$ is a convex problem with a second-order-cone epigraph.
\end{lem}

\begin{proof}
For a fixed competing class $k$, the perturbation adds $(\boldsymbol{\delta}_{t,k} - \boldsymbol{\delta}_{t,y_t})^\top \boldsymbol{\beta}_t = \langle \boldsymbol{\Delta}_t, \mathbf{M}_k \rangle$, where $\boldsymbol{\delta}_{t,k}^\top$ is row $k$ of $\boldsymbol{\Delta}_t$ and $\mathbf{M}_k$ has row $k$ equal to $\boldsymbol{\beta}_t^\top$, row $y_t$ equal to $-\boldsymbol{\beta}_t^\top$, and all other rows equal to zero. Since $\|\mathbf{M}_k\|_F = \sqrt{2}\|\boldsymbol{\beta}_t\|_2$, the support function of the Frobenius ball gives $\max_{\boldsymbol{\Delta}_t \in \mathcal{U}_t} \langle \boldsymbol{\Delta}_t, \mathbf{M}_k \rangle = \sqrt{2}\, r_t \|\boldsymbol{\beta}_t\|_2$. Maximization over the finite set of classes commutes with maximization over $\boldsymbol{\Delta}_t$, and the positive part is monotone, which proves the identity. Because $\boldsymbol{\beta}_t$ is affine in $(\boldsymbol{\beta}_0, \mathbf{V}_0)$, each term is convex and its norm epigraph is a second-order-cone constraint.
\end{proof}

If the ensemble weights must be valid voting weights, we additionally require $\boldsymbol{\beta}_t \geq \boldsymbol{0}$ and $\boldsymbol{e}^\top \boldsymbol{\beta}_t = 1$. A uniform guarantee requires a bounded set of admissible histories. For the Euclidean ball $\mathcal{Z} = \{\mathbf{z} : \|\mathbf{z} - \bar{\mathbf{z}}\|_2 \leq R\}$ with $R > 0$, write $\mathbf{v}_j^\top$ for row $j$ of $\mathbf{V}_0$. The voting constraints hold for every $\mathbf{z} \in \mathcal{Z}$ if and only if
\[
\boldsymbol{e}^\top \boldsymbol{\beta}_0 = 1, \qquad \boldsymbol{e}^\top \mathbf{V}_0 = \boldsymbol{0}, \qquad \beta_{0j} + \mathbf{v}_j^\top \bar{\mathbf{z}} \geq R \|\mathbf{v}_j\|_2 \quad (j = 1, \ldots, m).
\]
Indeed, the minimum of $\beta_{0j} + \mathbf{v}_j^\top \mathbf{z}$ over $\mathcal{Z}$ equals $\beta_{0j} + \mathbf{v}_j^\top \bar{\mathbf{z}} - R\|\mathbf{v}_j\|_2$, and $\boldsymbol{e}^\top(\boldsymbol{\beta}_0 + \mathbf{V}_0 \mathbf{z})$ is constant on a ball with nonempty interior only if $\boldsymbol{e}^\top \mathbf{V}_0 = \boldsymbol{0}$. The last inequalities are again second-order-cone representable, so robust adaptive weighted voting remains tractable under this bounded-history model. We establish the formulation here and leave its empirical evaluation on classification tasks to future work.

\section{Additional details on synthetic experiments}\label{app:synthetic}

\subsection{Data generation and disturbance scenarios}\label{app:scenarios}

Throughout, $\mathcal{N}(\mu, \sigma)$ denotes a normal distribution with mean $\mu$ and standard deviation $\sigma$, and $\mathcal{N}(\mathbf{b}, \boldsymbol{\Sigma})$ a multivariate normal distribution with mean $\mathbf{b}$ and covariance matrix $\boldsymbol{\Sigma}$. All synthetic experiments use the generator of Section~\ref{sec:synthetic}: the target $y(t)=\sin(2\pi t/500)+\epsilon(t)$ with $\epsilon(t) \sim \mathcal{N}(0, 0.1)$ over $T = 4{,}000$ time steps, and members $\tilde{X}_k(t) = y(t)+\epsilon_k(t)$ with $\epsilon_k(t) \sim \mathcal{N}(b_k, \sigma_k)$, $b_k \sim \operatorname{Uniform}(-0.5, 0.5)$, and $\sigma_k \sim \operatorname{Uniform}(0, 0.5)$. Each final member forecast is $X_k(t) = \tilde{X}_k(t)+\delta_k(t)$, where the disturbance $\delta_k(t)$ defines the scenario:
\begin{itemize}
\item \emph{No drift:} $\delta_k(t) = 0$.
\item \emph{Gradual drift (reference):} $\delta_k(t) = \frac{t}{T}\,\text{drift}_k(t)$ with $\text{drift}_k(t) \sim \mathcal{N}(b'_k, \sigma'_k)$, $b'_k \sim \mathcal{N}(0, \sigma_{\text{drift}})$, and $\sigma'_k \sim \operatorname{Uniform}(0, s_{\text{drift}})$. The sensitivity experiments of Figure~\ref{fig:synthetic_summary} vary $\sigma_{\text{drift}} = s_{\text{drift}}$ from 0 to 1, and all other experiments use 0.5.
\item \emph{Sudden regime shift:} $\delta_k(t) = c_k \mathbf{1}\{t \geq \lfloor 0.85 T \rfloor\}$ with $c_k \sim \mathcal{N}(0, 0.5)$; the shift begins inside the test segment, after model selection.
\item \emph{Periodic drift:} $\delta_k(t) = A_k \sin(2\pi t/250 + \phi_k)$ with $A_k \sim \operatorname{Uniform}(0, 0.5)$ and $\phi_k \sim \operatorname{Uniform}(0, 2\pi)$.
\item \emph{Persistent bias:} $\delta_k(t) = c_k\, t/T$ with $c_k \sim \operatorname{Uniform}(-0.5, 0.5)$.
\item \emph{Random-member degradation:} for a uniformly selected subset of 30\% of the members, $\delta_k(t) = \frac{t}{T}\, a_k \xi_k(t)$ with $a_k \sim \operatorname{Uniform}(0.5, 1)$ and $\xi_k(t) \sim \mathcal{N}(0, 1)$; the other members are unchanged.
\item \emph{Initially-best-member degradation:} the same increasing-noise disturbance is applied only to the member minimizing the pre-drift quality score $|b_k| + \sigma_k$.
\item \emph{Highly correlated members:} the member noise vector $\boldsymbol{\epsilon}(t) = (\epsilon_1(t), \ldots, \epsilon_m(t))$ follows $\mathcal{N}(\mathbf{b}, \mathbf{D}\mathbf{C}\mathbf{D})$ with $\mathbf{b} = (b_1, \ldots, b_m)$, $\mathbf{D} = \operatorname{Diag}(\sigma_1, \ldots, \sigma_m)$, $C_{kk} = 1$, and $C_{kj} = 0.9$ for $k \neq j$, followed by the gradual-drift disturbance.
\end{itemize}
Table~\ref{tab:features} lists the fixed values and ranges of all generator parameters in the sensitivity experiments. In all synthetic experiments, $m = 10$ unless it is varied, and we generated both the target and the member randomness independently for each of 30 seeds. Within a seed and scenario, every method uses exactly the same realization, which makes all comparisons paired.

\begin{table}[h]
\caption{Generator parameters of the sensitivity experiments, with their fixed values and the ranges over which they vary in Figure~\ref{fig:synthetic_summary}. The observations before testing are split 2:1 between training and validation, and except when the window size is varied, Adaptive Ridge selects $\tau \in \{2,3,4,5\}$ on validation.}
\centering
\small
\setlength{\tabcolsep}{4pt}
\begin{tabular}{llcl}
\toprule
\textbf{Parameter} & \textbf{Notation} & \textbf{Fixed value} & \textbf{Range when varied} \\
\midrule
Total number of time steps & $T$ & 4,000 & fixed \\
Observations before testing & & 3,000 & $\{50, 100, 200, 500, 1{,}000, 2{,}000\}$ \\
Test samples & & 1,000 & fixed \\
Seeds per configuration & & 30 & fixed \\
\midrule
Period of the target signal & & 500 & fixed \\
Target noise mean and standard deviation & & 0 and 0.1 & fixed \\
\midrule
Number of ensemble members & $m$ & 10 & $\{2, 5, 10, 15, 20, 25, 30\}$ \\
Range of member biases & $b_k$ & $(-0.5, 0.5)$ & fixed \\
Range of member noise levels & $\sigma_k$ & $(0, 0.5)$ & fixed \\
\midrule
Mean of the drift biases $b'_k$ & & 0 & fixed \\
Standard deviation of the drift biases & $\sigma_{\text{drift}}$ & 0.5 & $\{0, 0.25, 0.5, 0.75, 1\}$ \\
Upper bound of the drift noise levels $\sigma'_k$ & $s_{\text{drift}}$ & 0.5 & equal to $\sigma_{\text{drift}}$ \\
\midrule
Window size of Adaptive Ridge & $\tau$ & validated & $\{1, 2, 3, 4, 5, 10, 15, 25\}$ \\
\bottomrule
\end{tabular}\label{tab:features}
\end{table}

\FloatBarrier
\subsection{Sensitivity to the problem dimensions}\label{app:sweeps}

Figures~\ref{fig:sweeps_RMSE} to~\ref{fig:sweeps_CVAR_15} extend Figure~\ref{fig:synthetic_summary} to all eight methods of the sensitivity experiments and to all five metrics. These experiments use the common protocol of Appendix~\ref{app:protocol} and the grids of Table~\ref{tab:grids}; the window panel fixes $\tau$ and re-selects only the regularization on validation, using the results of Appendix~\ref{app:sensitivity}. When the number of members varies, the selected window lies at the edge of the grid: $\tau = 5$ for most seeds with two members and $\tau = 2$ for most seeds with 15 or more. Adding $\tau = 1$ to the grid for 15 to 30 members lowers the mean RMSE of Adaptive Ridge (from 0.0600 to 0.0579 at 25 members, for example) but leaves it above rolling ridge and RLS (0.0532), so the gap for large ensembles is not an artifact of the grid. Table~\ref{tab:sweeps_sd} gives the standard deviation of test RMSE across seeds for each point of Figure~\ref{fig:sweeps_RMSE}; as in Appendix~\ref{app:complete}, most of this spread is shared by all methods.

\begin{figure}[h]
\centering
\includegraphics[width=\linewidth]{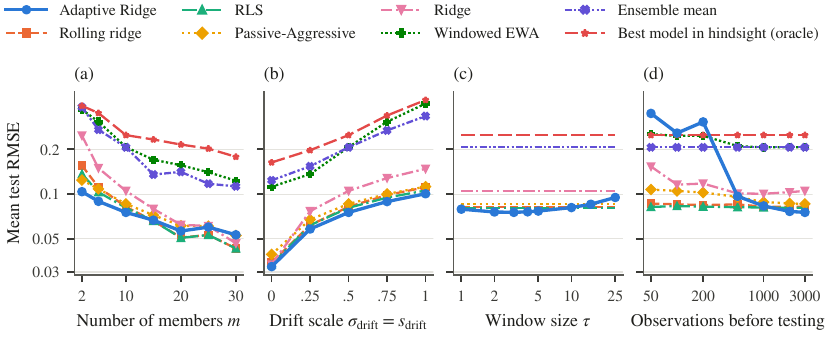}
\caption{Mean test RMSE over 30 seeds for the four sensitivity experiments of Figure~\ref{fig:synthetic_summary}, for all eight methods.}
\label{fig:sweeps_RMSE}
\end{figure}

\begin{figure}[h]
\centering
\includegraphics[width=\linewidth]{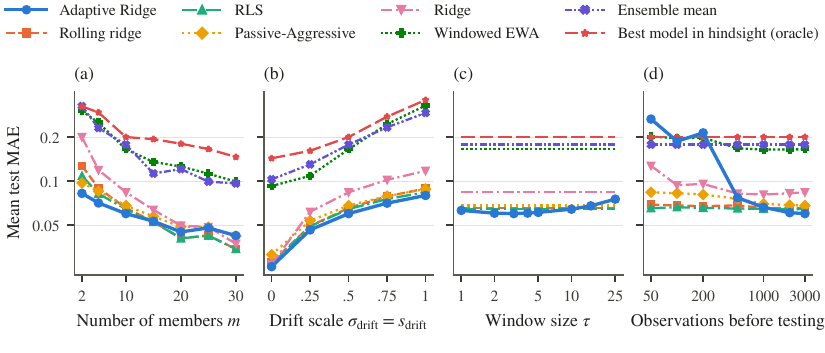}
\caption{Mean test MAE over 30 seeds for the four sensitivity experiments of Figure~\ref{fig:synthetic_summary}, for all eight methods.}
\label{fig:sweeps_MAE}
\end{figure}

\begin{figure}[h]
\centering
\includegraphics[width=\linewidth]{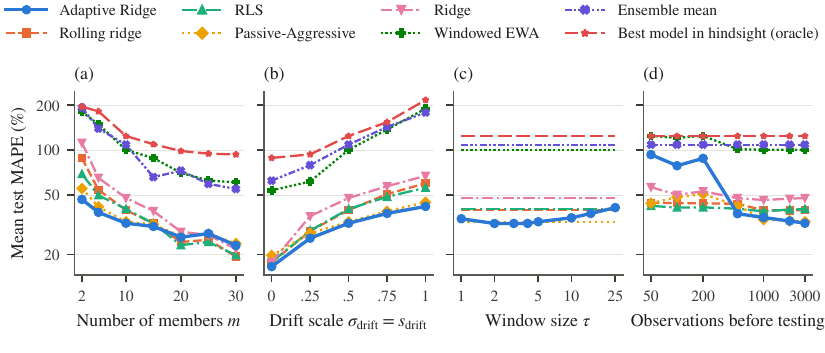}
\caption{Mean test MAPE over 30 seeds for the four sensitivity experiments of Figure~\ref{fig:synthetic_summary}, for all eight methods.}
\label{fig:sweeps_MAPE}
\end{figure}

\begin{figure}[h]
\centering
\includegraphics[width=\linewidth]{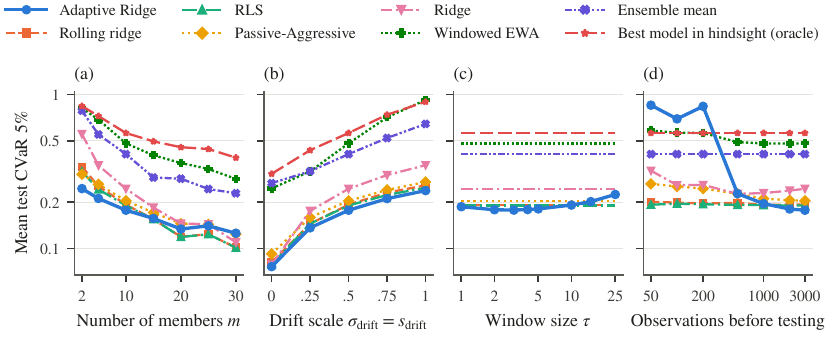}
\caption{Mean test CVaR 5\% over 30 seeds for the four sensitivity experiments of Figure~\ref{fig:synthetic_summary}, for all eight methods.}
\label{fig:sweeps_CVAR_05}
\end{figure}

\begin{figure}[h]
\centering
\includegraphics[width=\linewidth]{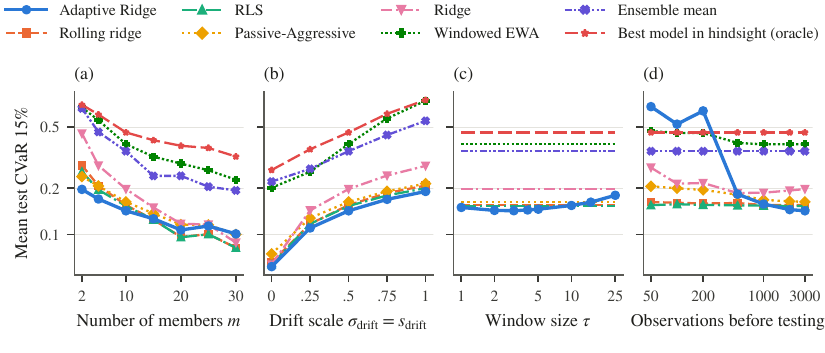}
\caption{Mean test CVaR 15\% over 30 seeds for the four sensitivity experiments of Figure~\ref{fig:synthetic_summary}, for all eight methods.}
\label{fig:sweeps_CVAR_15}
\end{figure}

\begin{table}[h]
\caption{Standard deviation of test RMSE across the 30 seeds for each point of Figure~\ref{fig:sweeps_RMSE}. When the window size is varied, only Adaptive Ridge changes; the dashes stand for the other methods, whose values are those of the drift scale 0.5 row. EWA: windowed EWA; Mean: ensemble mean; Oracle: best model in hindsight.}
\centering
\small
\setlength{\tabcolsep}{4pt}
\begin{tabular}{lcccccccc}
\toprule
\textbf{Setting} & \textbf{Adaptive Ridge} & \textbf{Rolling ridge} & \textbf{RLS} & \textbf{PA} & \textbf{Ridge} & \textbf{EWA} & \textbf{Mean} & \textbf{Oracle} \\
\midrule
\multicolumn{9}{l}{\textit{Number of members $m$}} \\
2 & $.0157$ & $.0381$ & $.0305$ & $.0249$ & $.0893$ & $.1507$ & $.1752$ & $.1457$ \\
5 & $.0264$ & $.0392$ & $.0354$ & $.0351$ & $.0494$ & $.1060$ & $.1139$ & $.1199$ \\
10 & $.0249$ & $.0338$ & $.0333$ & $.0256$ & $.0303$ & $.0872$ & $.0759$ & $.1056$ \\
15 & $.0157$ & $.0176$ & $.0181$ & $.0130$ & $.0231$ & $.0617$ & $.0439$ & $.0881$ \\
20 & $.0141$ & $.0164$ & $.0173$ & $.0113$ & $.0161$ & $.0508$ & $.0459$ & $.0738$ \\
25 & $.0097$ & $.0125$ & $.0134$ & $.0088$ & $.0132$ & $.0624$ & $.0386$ & $.0804$ \\
30 & $.0097$ & $.0132$ & $.0135$ & $.0086$ & $.0142$ & $.0454$ & $.0317$ & $.0706$ \\
\midrule
\multicolumn{9}{l}{\textit{Drift scale $\sigma_{\text{drift}} = s_{\text{drift}}$}} \\
0 & $.0202$ & $.0218$ & $.0224$ & $.0221$ & $.0224$ & $.0631$ & $.0401$ & $.0842$ \\
0.25 & $.0209$ & $.0262$ & $.0270$ & $.0217$ & $.0257$ & $.0613$ & $.0443$ & $.0865$ \\
0.5 & $.0249$ & $.0338$ & $.0333$ & $.0256$ & $.0303$ & $.0872$ & $.0759$ & $.1056$ \\
0.75 & $.0282$ & $.0387$ & $.0365$ & $.0287$ & $.0363$ & $.1158$ & $.1148$ & $.1163$ \\
1 & $.0307$ & $.0407$ & $.0383$ & $.0301$ & $.0427$ & $.1386$ & $.1537$ & $.1411$ \\
\midrule
\multicolumn{9}{l}{\textit{Window size $\tau$}} \\
1 & $.0258$ & -- & -- & -- & -- & -- & -- & -- \\
2 & $.0249$ & -- & -- & -- & -- & -- & -- & -- \\
3 & $.0249$ & -- & -- & -- & -- & -- & -- & -- \\
4 & $.0254$ & -- & -- & -- & -- & -- & -- & -- \\
5 & $.0257$ & -- & -- & -- & -- & -- & -- & -- \\
10 & $.0260$ & -- & -- & -- & -- & -- & -- & -- \\
15 & $.0251$ & -- & -- & -- & -- & -- & -- & -- \\
25 & $.0205$ & -- & -- & -- & -- & -- & -- & -- \\
\midrule
\multicolumn{9}{l}{\textit{Observations before testing}} \\
50 & $.1108$ & $.0339$ & $.0322$ & $.0221$ & $.0676$ & $.1071$ & $.0759$ & $.1056$ \\
100 & $.0754$ & $.0345$ & $.0313$ & $.0191$ & $.0427$ & $.1055$ & $.0759$ & $.1056$ \\
200 & $.0794$ & $.0348$ & $.0312$ & $.0197$ & $.0471$ & $.1049$ & $.0759$ & $.1056$ \\
500 & $.0244$ & $.0341$ & $.0326$ & $.0293$ & $.0264$ & $.0887$ & $.0759$ & $.1056$ \\
1{,}000 & $.0262$ & $.0339$ & $.0338$ & $.0226$ & $.0271$ & $.0875$ & $.0759$ & $.1056$ \\
2{,}000 & $.0251$ & $.0338$ & $.0332$ & $.0248$ & $.0296$ & $.0871$ & $.0759$ & $.1056$ \\
3{,}000 & $.0249$ & $.0338$ & $.0333$ & $.0256$ & $.0303$ & $.0872$ & $.0759$ & $.1056$ \\
\bottomrule
\end{tabular}\label{tab:sweeps_sd}
\end{table}

\FloatBarrier
\subsection{Methods and evaluation protocol}\label{app:protocol}

Both the sensitivity experiments of Figure~\ref{fig:synthetic_summary} and the comparison across disturbance scenarios follow the common protocol described here. The sensitivity experiments use eight methods (the core methods of Table~\ref{tab:ensemble_methods}, rolling ridge, and RLS), and the comparison across disturbance scenarios uses 15. Two exceptions apply to the sensitivity experiments: when the number of observations before testing is varied, the training and validation segments are shorter (Figure~\ref{fig:synthetic_summary}d), and the best model in hindsight, shown only as a reference, is chosen on the test segment.

\begin{table}[h]
\caption{Summary of the core ensemble methods compared.}
\centering
\small
\setlength{\tabcolsep}{4pt}
\begin{tabular}{lll}
\toprule
\textbf{Method} & \textbf{Time-varying} & \textbf{Hyperparameters to tune} \\
\midrule
Best Model in Hindsight (oracle) & No & N/A \\
Ensemble Mean & No & N/A \\
Windowed EWA & Yes & Window of past data to use to compute regrets \\
Passive-Aggressive & Yes & Margin parameter $\epsilon$ \\
Ridge & No & Regularization factor $\lambda$ \\
Adaptive Ridge & Yes & Regularization factor $\lambda$, window of past data in $\mathbf{Z}_t$ \\
\bottomrule
\end{tabular}\label{tab:ensemble_methods}
\end{table}

For every scenario and seed, all methods receive identical member forecasts and targets and use the same chronological split of 2,000 training, 1,000 validation, and 1,000 test observations. Preprocessing and candidate fitting use training observations only, and hyperparameters are selected solely by validation RMSE. Preprocessing is then recomputed on training plus validation, selected batch models are refitted and online methods are replayed on those observations, and the test segment is evaluated once. At validation and test times, updates and error-history features use an outcome only after its prediction has been made. All methods therefore share the same information timeline, and no test target is used for model selection.

Besides Adaptive Ridge, the ensemble mean, and static ridge (Ridge in Table~\ref{tab:ensemble_methods}), the 15 methods include the following. The validation-selected member is the single member with the lowest validation RMSE; unlike the best model in hindsight, it is deployable. Windowed EWA computes exponential weights from member losses over a window of past observations, and Hedge \citep{freund1997decision} uses cumulative losses with a tuned learning rate. Fixed Share \citep{herbster1998tracking} redistributes a share of the exponentially updated expert weights at every step to track switches in the best member. Passive-Aggressive I (PA-I) \citep{passive} performs capped $\epsilon$-insensitive online linear-regression updates. Recursive least squares (RLS) \citep{haykin2002adaptive} recursively updates linear weights, including a constant term, with a forgetting factor. Rolling ridge refits ridge weights on the most recent $W$ labeled observations, and discounted ridge refits them with exponentially decaying observation weights. The regression-tree and gradient-boosted regression-tree (GBRT) \citep{friedman2001greedy} stackers map the current member forecasts directly to the target without error-feedback weights. Two methods represent time-varying meta-learning. Our reimplementation of the arbitrated dynamic ensemble (ADE) \citep{cerqueira2019arbitrage} fits one random-forest arbiter per member to predict that member's next absolute error from strictly past member-error histories; at each time, recent observed mean absolute error determines committee membership, and the predicted errors are aggregated into weights. The direct-loss softmax gate maps the same past-error histories directly to nonnegative weights summing to one and is trained end-to-end on the squared error of the combined forecast. Table~\ref{tab:grids} lists the complete validation grids. Grid sizes differ across methods and are reported rather than forced to be equal.

\begin{table}[h]
\caption{Validation grids for all synthetic experiments. Each method selects one configuration by validation RMSE.}
\centering
\footnotesize
\setlength{\tabcolsep}{4pt}
\begin{tabular}{llc}
\toprule
\textbf{Method} & \textbf{Validation candidates} & \textbf{Size} \\
\midrule
Adaptive Ridge & $\tau \in \{2,3,4,5\}$; regularization $\in \{10^{-4},10^{-3},10^{-2},10^{-1},1\}$ & 20 \\
Validation-selected member & one candidate per member & 10 \\
Windowed EWA & $W \in \{10, 50, 100, 200, 500, \text{full history}\}$ & 6 \\
Hedge & learning rate $\in \{\text{anytime}, 0.01, 0.05, 0.1, 0.5\}$ & 5 \\
Fixed Share & learning rate $\in \{0.01, 0.05, 0.1, 0.5, 1\}$; share rate $\in \{0, 10^{-3}, 0.01, 0.05, 0.1\}$ & 25 \\
PA-I & $\epsilon \in \{0, 0.01, 0.05, 0.1, 0.5\}$; $C \in \{0.01, 0.1, 1, 10\}$ & 20 \\
RLS & forgetting factor $\in \{0.95, 0.98, 0.99, 0.999, 1\}$ & 5 \\
Static ridge & regularization $\in \{10^{-4},10^{-3},10^{-2},10^{-1},1\}$ & 5 \\
Rolling ridge & $W \in \{50, 100, 200, 500\}$; regularization $\in \{10^{-3},10^{-2},10^{-1}\}$ & 12 \\
Discounted ridge & discount $\in \{0.99, 0.995, 0.999, 1\}$; regularization $\in \{10^{-3},10^{-2},10^{-1}\}$ & 12 \\
Regression-tree stacker & depth $\in \{2, 3\}$ & 2 \\
GBRT stacker & depth $\in \{2, 3\}$; 50 trees; learning rate $\in \{0.05, 0.1\}$ & 4 \\
Direct-loss softmax gate & $\tau \in \{2, 5, 10\}$; regularization $\in \{0, 10^{-4}, 10^{-3}, 10^{-2}\}$ & 12 \\
ADE-style arbitration & $\tau \in \{5, 10\}$; committee window $\in \{20, 100\}$; retained fraction $\in \{0.5, 1\}$; & 32 \\
 & linear or softmax aggregation; correlation reweighting off or on & \\
\bottomrule
\end{tabular}\label{tab:grids}
\end{table}

\FloatBarrier
\subsection{Results across disturbance scenarios}\label{app:complete}

Table~\ref{tab:complete} reports mean test RMSE for all 15 methods. Adaptive Ridge has lower mean RMSE than Hedge, static ridge, discounted ridge, the regression-tree and GBRT stackers, and the validation-selected member in all eight scenarios; than PA-I, rolling ridge, Fixed Share, windowed EWA, the softmax gate, ADE-style arbitration, and the ensemble mean in seven scenarios; and than RLS in six scenarios. In every scenario except the sudden shift, it has lower RMSE than the softmax gate on at least 29 of 30 seeds and than ADE-style arbitration on all 30 seeds, and these differences remain significant after Holm correction across the eight scenarios of each comparison. The primary confirmatory comparison is the one against rolling ridge in Figure~\ref{fig:paired}, which uses two-sided paired Wilcoxon signed-rank tests \citep{wilcoxon1945individual} with Holm correction \citep{holm1979simple}; the remaining comparisons are descriptive. Table~\ref{tab:complete_sd} gives the standard deviation of test RMSE across seeds. Most of this spread is shared by all methods, because some generated series are harder than others: under gradual drift, the standard deviation is 0.0249 for Adaptive Ridge and 0.0338 for rolling ridge, but only 0.0104 for their seed-by-seed difference. Figure~\ref{fig:paired} therefore compares the methods within each seed.

\begin{table}[h]
\caption{Mean test RMSE over 30 paired seeds for all 15 methods across disturbance scenarios. Bold marks the lowest value in each column. Abbreviations: Grad.\ = gradual drift, Per.\ = periodic drift, Bias = persistent bias, Rand.\ = random-member degradation, Best = initially-best-member degradation, Corr.\ = highly correlated members.}
\centering
\small
\setlength{\tabcolsep}{4.5pt}
\begin{tabular}{lcccccccc}
\toprule
\textbf{Method} & \textbf{None} & \textbf{Grad.} & \textbf{Sudden} & \textbf{Per.} & \textbf{Bias} & \textbf{Rand.} & \textbf{Best} & \textbf{Corr.} \\
\midrule
Adaptive Ridge & $\mathbf{.0326}$ & $\mathbf{.0753}$ & $.2075$ & $\mathbf{.0346}$ & $.0353$ & $\mathbf{.0497}$ & $\mathbf{.0453}$ & $\mathbf{.0854}$ \\
Ensemble mean & $.1236$ & $.2059$ & $.1770$ & $.1355$ & $.1448$ & $.1731$ & $.1431$ & $.3130$ \\
Validation-selected member & $.1630$ & $.2600$ & $.4044$ & $.2225$ & $.1641$ & $.1849$ & $.2474$ & $.2929$ \\
Windowed EWA & $.1120$ & $.2065$ & $.1489$ & $.1259$ & $.1220$ & $.1490$ & $.1670$ & $.2473$ \\
Hedge & $.1334$ & $.2238$ & $.2588$ & $.1904$ & $.1547$ & $.1637$ & $.2015$ & $.2700$ \\
Fixed Share & $.0832$ & $.1356$ & $.1125$ & $.0971$ & $.0938$ & $.1046$ & $.1070$ & $.2197$ \\
PA-I & $.0394$ & $.0859$ & $\mathbf{.0478}$ & $.0498$ & $.0403$ & $.0563$ & $.0528$ & $.0964$ \\
RLS & $.0337$ & $.0809$ & $.0722$ & $.0484$ & $\mathbf{.0347}$ & $.0504$ & $.0471$ & $.0968$ \\
Static ridge & $.0337$ & $.1050$ & $.2238$ & $.0531$ & $.0539$ & $.0547$ & $.0499$ & $.1219$ \\
Rolling ridge & $.0348$ & $.0821$ & $.0495$ & $.0493$ & $.0358$ & $.0513$ & $.0483$ & $.0997$ \\
Discounted ridge & $.0345$ & $.0953$ & $.2113$ & $.0537$ & $.0453$ & $.0523$ & $.0485$ & $.1120$ \\
Regression-tree stacker & $.0993$ & $.2158$ & $.3003$ & $.1191$ & $.1468$ & $.1191$ & $.1182$ & $.2213$ \\
GBRT stacker & $.0384$ & $.1192$ & $.2286$ & $.0549$ & $.0658$ & $.0565$ & $.0535$ & $.1404$ \\
Direct-loss softmax gate & $.0504$ & $.1073$ & $.1966$ & $.0551$ & $.0577$ & $.0804$ & $.0738$ & $.1669$ \\
ADE-style arbitration & $.0873$ & $.1355$ & $.1566$ & $.0978$ & $.0966$ & $.1135$ & $.1211$ & $.2243$ \\
\bottomrule
\end{tabular}\label{tab:complete}
\end{table}

\begin{table}[h]
\caption{Standard deviation of test RMSE across the 30 seeds for the methods and scenarios of Table~\ref{tab:complete}.}
\centering
\small
\setlength{\tabcolsep}{4.5pt}
\begin{tabular}{lcccccccc}
\toprule
\textbf{Method} & \textbf{None} & \textbf{Grad.} & \textbf{Sudden} & \textbf{Per.} & \textbf{Bias} & \textbf{Rand.} & \textbf{Best} & \textbf{Corr.} \\
\midrule
Adaptive Ridge & $.0202$ & $.0249$ & $.1264$ & $.0204$ & $.0214$ & $.0314$ & $.0268$ & $.0214$ \\
Ensemble mean & $.0401$ & $.0759$ & $.0678$ & $.0383$ & $.0525$ & $.0303$ & $.0357$ & $.0748$ \\
Validation-selected member & $.0842$ & $.1104$ & $.2664$ & $.0919$ & $.1150$ & $.1108$ & $.1028$ & $.0944$ \\
Windowed EWA & $.0631$ & $.0872$ & $.0700$ & $.0516$ & $.0811$ & $.0858$ & $.0733$ & $.0614$ \\
Hedge & $.0810$ & $.0752$ & $.1476$ & $.0845$ & $.1111$ & $.0912$ & $.1009$ & $.0768$ \\
Fixed Share & $.0356$ & $.0360$ & $.0376$ & $.0279$ & $.0477$ & $.0429$ & $.0400$ & $.0546$ \\
PA-I & $.0221$ & $.0256$ & $.0223$ & $.0198$ & $.0222$ & $.0344$ & $.0299$ & $.0263$ \\
RLS & $.0224$ & $.0333$ & $.0501$ & $.0216$ & $.0228$ & $.0349$ & $.0315$ & $.0339$ \\
Static ridge & $.0224$ & $.0303$ & $.1499$ & $.0254$ & $.0351$ & $.0413$ & $.0361$ & $.0506$ \\
Rolling ridge & $.0218$ & $.0338$ & $.0282$ & $.0209$ & $.0218$ & $.0349$ & $.0318$ & $.0363$ \\
Discounted ridge & $.0217$ & $.0285$ & $.1425$ & $.0249$ & $.0236$ & $.0365$ & $.0323$ & $.0422$ \\
Regression-tree stacker & $.0198$ & $.0679$ & $.1639$ & $.0213$ & $.0415$ & $.0527$ & $.0542$ & $.0625$ \\
GBRT stacker & $.0221$ & $.0341$ & $.1354$ & $.0213$ & $.0375$ & $.0380$ & $.0354$ & $.0467$ \\
Direct-loss softmax gate & $.0265$ & $.0303$ & $.1383$ & $.0241$ & $.0345$ & $.0425$ & $.0430$ & $.0524$ \\
ADE-style arbitration & $.0312$ & $.0305$ & $.0584$ & $.0244$ & $.0388$ & $.0449$ & $.0318$ & $.0552$ \\
\bottomrule
\end{tabular}\label{tab:complete_sd}
\end{table}

\subsection{Sensitivity to the window size and regularization}\label{app:sensitivity}

Using the gradual-drift process, the same chronological split, and 30 independently generated series, we evaluated a broader grid than the one used for model selection: $\tau \in \{1, 2, 3, 4, 5, 10, 15, 25\}$ and regularization in $\{10^{-4}, 10^{-3}, 10^{-2}, 10^{-1}, 1\}$, for 40 configurations per seed. For each seed, all configurations were fitted on the training segment, the configuration with the lowest validation RMSE was refitted on training plus validation, and the test segment was evaluated once. Table~\ref{tab:sensitivity} reports the mean test RMSE of every fixed configuration; these entries are diagnostic and were not used for model selection. The last column and the last row (``Selected'') report how often each value appeared in the configuration selected by validation RMSE. After validation selection, the test RMSE was $0.0753 \pm 0.0249$ (mean $\pm$ standard deviation), the same mean as with the grid of Table~\ref{tab:grids}; the selected configuration differs only in the 2 of 30 seeds where $\tau = 1$ was chosen. The surface is relatively flat for $\tau \in \{1, \ldots, 5\}$ under light regularization, whereas long windows require stronger regularization and never enter the selected configuration. Performance therefore does not hinge on a single setting, and the broader grid confirms the expected deterioration with overly long histories.

\begin{table}[h]
\caption{Mean test RMSE over 30 seeds for each fixed window size and regularization value under gradual drift, with the number of seeds for which each value was selected by validation RMSE.}
\centering
\small
\begin{tabular}{rcccccc}
\toprule
$\tau$ & $10^{-4}$ & $10^{-3}$ & $10^{-2}$ & $10^{-1}$ & $1$ & \textbf{Selected} \\
\midrule
1 & $.0794$ & $.0797$ & $.0838$ & $.0967$ & $.1319$ & 2/30 \\
2 & $.0773$ & $.0769$ & $.0798$ & $.0921$ & $.1236$ & 7/30 \\
3 & $.0787$ & $.0776$ & $.0790$ & $.0902$ & $.1198$ & 13/30 \\
4 & $.0808$ & $.0790$ & $.0791$ & $.0895$ & $.1179$ & 5/30 \\
5 & $.0836$ & $.0809$ & $.0798$ & $.0892$ & $.1167$ & 3/30 \\
10 & $.1028$ & $.0943$ & $.0861$ & $.0908$ & $.1155$ & 0/30 \\
15 & $.1297$ & $.1103$ & $.0928$ & $.0932$ & $.1163$ & 0/30 \\
25 & $.2225$ & $.1417$ & $.1042$ & $.0975$ & $.1183$ & 0/30 \\
\midrule
\textbf{Selected} & 5/30 & 8/30 & 16/30 & 1/30 & 0/30 & \\
\bottomrule
\end{tabular}\label{tab:sensitivity}
\end{table}

\subsection{Limited history and abrupt shifts}\label{app:history}

Limited history is one clear boundary of the method. We held the final 1,000-observation test segment of the gradual-drift process fixed and varied the total history immediately preceding it. For each history length, the observations were divided approximately 2:1 into chronological training and validation segments, with all hyperparameters selected on validation only. Table~\ref{tab:limited_history} reports mean test RMSE over 30 seeds. With very short histories, both coefficient estimation and validation selection have high variance: the 50-, 100-, and 200-observation conditions leave only 17, 34, and 67 validation observations, so the curve need not be monotone in this range. The main transition occurs between 200 and 1,000 observations, as the affine rule moves from unstable estimation to performance comparable to or better than the simpler online combiners. This pattern is consistent with the overfitting risk discussed in Section~\ref{sec:synthetic} and supports the recommendation that the method requires sufficient historical data.

An abrupt regime shift that appears only after model selection is a second failure case. In the sudden-shift scenario, Adaptive Ridge obtains a mean test RMSE of 0.2075, compared with 0.0495 for rolling ridge and 0.0478 for PA-I (Table~\ref{tab:complete}). Although its prediction-time weights vary with recent errors, the affine error-to-weight map is fitted before testing. A shift that first appears after model selection therefore produces error states that were absent from training and validation, whereas rolling ridge and online methods update their coefficients after the new outcomes arrive. The method is useful when the relationships between recent errors and reliable weights are represented in the historical data, but it is not designed for immediate adaptation to an entirely unseen regime.

\begin{table}[h]
\caption{Mean test RMSE over 30 seeds as a function of the total history before the fixed 1,000-observation test segment (gradual drift). Bold marks the lowest value in each row.}
\centering
\small
\begin{tabular}{rccccc}
\toprule
\textbf{Pre-test history} & \textbf{Adaptive Ridge} & \textbf{Rolling ridge} & \textbf{RLS} & \textbf{PA-I} & \textbf{Ensemble mean} \\
\midrule
50 & $.3476$ & $.0864$ & $\mathbf{.0819}$ & $.1073$ & $.2059$ \\
100 & $.2565$ & $.0854$ & $\mathbf{.0831}$ & $.1049$ & $.2059$ \\
200 & $.3048$ & $.0844$ & $\mathbf{.0824}$ & $.1024$ & $.2059$ \\
500 & $.0967$ & $.0852$ & $\mathbf{.0816}$ & $.0957$ & $.2059$ \\
1,000 & $.0833$ & $.0823$ & $\mathbf{.0813}$ & $.0884$ & $.2059$ \\
2,000 & $\mathbf{.0765}$ & $.0822$ & $.0808$ & $.0865$ & $.2059$ \\
\bottomrule
\end{tabular}\label{tab:limited_history}
\end{table}

\subsection{Correlated ensemble members}\label{app:correlation}

To examine correlated members beyond a single setting, we used the same $T = 4{,}000$, $m = 10$, chronological split, 30 seeds, and complete validation grids. Conditional on the member scales of each seed, the member noise vector follows
\[
\boldsymbol{\epsilon}(t) \sim \mathcal{N}\big(\mathbf{b}, \mathbf{D}\mathbf{C}(r)\mathbf{D}\big), \qquad C_{kk}(r) = 1, \quad C_{kj}(r) = r \ (k \neq j),
\]
with $\mathbf{b}$ and $\mathbf{D}$ as in Appendix~\ref{app:scenarios} and $r \in \{0, 0.25, 0.5, 0.75, 0.9, 0.99\}$. We then apply the same member-specific gradual-drift disturbance as in the reference scenario, so $r$ controls the correlation among the base errors while leaving the drift process unchanged. Table~\ref{tab:correlation} reports the results. Adaptive Ridge has the lowest mean RMSE at every evaluated correlation. Against rolling ridge, it wins on 19, 26, 26, 23, 22, and 17 of 30 paired seeds as $r$ increases, and paired two-sided Wilcoxon tests remain significant after Holm correction across the six correlation levels (largest adjusted $p = 0.0145$). The ensemble mean deteriorates monotonically as the shared error component grows. The learned linear combiners improve again near $r = 1$; this non-monotone pattern is consistent with an almost low-rank shared error component and heterogeneous member scales, which allow regularized signed weights to cancel part of that component. Correlated base errors alone are therefore not a failure case in this controlled process.

\begin{table}[h]
\caption{Mean test RMSE over 30 paired seeds as a function of the base-error correlation $r$ (gradual drift). Bold marks the lowest value in each row.}
\centering
\small
\begin{tabular}{rccccc}
\toprule
$r$ & \textbf{Adaptive Ridge} & \textbf{Rolling ridge} & \textbf{RLS} & \textbf{PA-I} & \textbf{Ensemble mean} \\
\midrule
0 & $\mathbf{.0753}$ & $.0821$ & $.0809$ & $.0859$ & $.2059$ \\
0.25 & $\mathbf{.0885}$ & $.1042$ & $.1015$ & $.0999$ & $.2408$ \\
0.5 & $\mathbf{.0909}$ & $.1091$ & $.1053$ & $.1031$ & $.2712$ \\
0.75 & $\mathbf{.0896}$ & $.1064$ & $.1027$ & $.1009$ & $.2981$ \\
0.9 & $\mathbf{.0854}$ & $.0997$ & $.0968$ & $.0964$ & $.3130$ \\
0.99 & $\mathbf{.0780}$ & $.0866$ & $.0842$ & $.0899$ & $.3217$ \\
\bottomrule
\end{tabular}\label{tab:correlation}
\end{table}

\subsection{Computational cost}\label{app:runtime}

We report two sets of timings. The first compares Adaptive Ridge with nine baselines under the common protocol with the closed-form implementation of Appendix~\ref{app:implementation}. The second describes how the training time of the solver-based implementation grows with the data size and the window size.

Table~\ref{tab:runtime} reports median runtimes over seven repetitions after one untimed warm-up run, on an Apple M2 with Julia 1.11.3, using one Julia thread and one BLAS thread, the same 2,000/1,000/1,000 split, and identical conditions for every method. The number in parentheses after each method is the number of hyperparameter configurations evaluated during validation. Validation search is the total time for the complete grid. Final refit or replay fits the selected batch model on training plus validation, or reconstructs the selected online method's state before the test segment. The test phase is one untouched predict-then-update pass over 1,000 observations; rolling ridge has no separate refit phase because its refits are counted during test prediction. Compilation and common preprocessing are excluded.

\begin{table}[h]
\caption{Median runtimes in seconds over seven repetitions. Numbers in parentheses give the size of each validation grid.}
\centering
\small
\begin{tabular}{lcccc}
\toprule
\textbf{Method (grid size)} & \textbf{Validation search} & \textbf{Final refit or replay} & \textbf{1,000-step test} & \textbf{Total} \\
\midrule
Adaptive Ridge (20) & 0.1046 & 0.0081 & 0.0003 & 0.1135 \\
PA-I (20) & 0.0047 & 0.0002 & 0.0001 & 0.0050 \\
RLS (5) & 0.0101 & 0.0016 & 0.0006 & 0.0124 \\
Rolling ridge (12) & 0.0483 & n/a & 0.0077 & 0.0560 \\
Windowed EWA (6) & 0.0007 & $<$0.0001 & 0.0001 & 0.0008 \\
Hedge (5) & 0.0009 & 0.0002 & 0.0001 & 0.0011 \\
Fixed Share (25) & 0.0083 & 0.0003 & 0.0001 & 0.0087 \\
Direct-loss softmax gate (12) & 0.9518 & 0.1580 & 0.0010 & 1.1073 \\
GBRT stacker (4) & 0.8967 & 0.3740 & 0.0004 & 1.2980 \\
ADE-style arbitration (32) & 4.0795 & 2.6726 & 0.0291 & 6.7804 \\
\bottomrule
\end{tabular}\label{tab:runtime}
\end{table}

Simple online methods are cheaper to select, while Adaptive Ridge remains below one second at this scale and is faster than the softmax gate, the GBRT stacker, and ADE-style arbitration. Its 1,000-step test pass takes 0.0003 seconds, so the cost of the double loop over window sizes and regularization values is concentrated in offline model selection rather than in per-prediction latency.

\paragraph{Larger histories.} Under the same protocol, we also evaluated scaling in the history length at fixed $m = 10$, using the complete 20-configuration grid and a fixed 1,000-step test segment. Increasing the pre-test history from 3,000 observations (2,000 training and 1,000 validation) to 100,000 observations (66,667 training and 33,333 validation) increased the median time of the complete validation search, final refit, and test pass from about 0.11 to 3.64 seconds, closely tracking the $33.3\times$ increase in history size. At 100,000 observations, the test pass alone required 0.00058 seconds.

\paragraph{Solver-based implementation.} The solver-based implementation, used for the real-world case studies, is written in Julia 1.6 \citep{bezanson2017julia} with the JuMP package \citep{jump} and the Gurobi solver \citep{gurobi}; these experiments ran on 4 Intel Xeon Platinum 8260 CPU cores of the MIT SuperCloud cluster \citep{reuther2018interactive}. Tables~\ref{tab:training_time} and~\ref{tab:past} report the training time of this implementation on synthetic data as we vary the number of training samples and the window size. Overall, the method remains tractable with typical values of those hyperparameters.

\begin{table}[h]
\caption{Average training time across 30 seeds of the solver-based implementation of Adaptive Ridge with respect to the number of samples used in the training set, when we fix the number of ensemble members to $m = 10$ and the number of past time steps to use to $\tau = 5$.}
\centering
\small
\begin{tabular}{cc}
\toprule
\textbf{Number of samples in training set} & \textbf{Average time (s)} \\
\midrule
100 & $0.4$ \\
200 & $0.6$ \\
400 & $1.3$ \\
750 & 64 \\
1500 & 92 \\
3000 & 239 \\
\bottomrule
\end{tabular}\label{tab:training_time}
\end{table}

\begin{table}[h]
\caption{Average training time across 30 seeds of the solver-based implementation of Adaptive Ridge with respect to the number of past time steps to use, when we fix the number of ensemble members to $m = 10$ and the number of training samples to $N = 3000$.}
\centering
\small
\begin{tabular}{cc}
\toprule
\textbf{Number of time steps} & \textbf{Average time (s)} \\
\midrule
2 & 54 \\
3 & 100 \\
5 & 239 \\
10 & 603 \\
15 & 1079 \\
25 & 1361 \\
\bottomrule
\end{tabular}\label{tab:past}
\end{table}

\clearpage
\section{Additional details on real-world experiments}\label{app:realworld}

\subsection{Hyperparameters in the case studies}\label{app:realworld_hp}

In all three case studies, we tuned the regularization factor of Ridge and Adaptive Ridge and the parameter $\epsilon$ of Passive-Aggressive over $\{0, 10^{-4}, 10^{-3}, 10^{-2}, 10^{-1}, 1, 2\}$, and the window sizes of Adaptive Ridge and windowed EWA over $\{1, \ldots, 10\}$. All three case studies split the data in time order, without shuffling. For Safi, the ensembles use the first 50\% of the observations for training, the next 20\% for validation, and the last 30\% for testing. For the energy data, the ensemble members were trained on the first 50\% of the series (9,858 observations); the ensembles were then trained on the next 25\% (4,929), validated on the next 10\% (1,973), and tested on the last 15\% (2,956), so every member forecast that the ensembles use is out of sample for that member. For tropical cyclones, the member forecasts are those of \citet{boussioux} for their test years, 2016 to 2019, which we split in time order into 50\% for training, 20\% for validation, and 30\% for testing.

\subsection{Safi and energy ensemble members}\label{app:realworld_members}

Table \ref{tab:safi_members} lists the six fixed forecast inputs used in the Safi wind-speed ensemble. Table \ref{tab:energy_members} lists the ten validation-selected appliance-energy forecasters used in the energy ensemble. In both cases, the Adaptive Ridge model only receives the issued forecasts and past forecast errors; it does not access the internal states of these member models.

\begin{table}[h]
\caption{Safi wind-speed ensemble members. All members provide one-hour-ahead wind-speed forecasts for the Safi site.}
\centering
\small
\begin{tabular}{lll}
\toprule
\textbf{Member ID} & \textbf{Forecast identity} & \textbf{Used as} \\
\midrule
Safi-1 & Moroccan weather agency operational forecast & Ensemble member \\
Safi-2 & XGBoost & Ensemble member \\
Safi-3 & Decision Trees & Ensemble member \\
Safi-4 & Optimal Regression Trees & Ensemble member \\
Safi-5 & Lasso Regression & Ensemble member \\
Safi-6 & Ridge Regression & Ensemble member \\
\bottomrule
\end{tabular}\label{tab:safi_members}
\end{table}

\begin{table}[h]
\caption{Energy-consumption ensemble members selected by validation MAE.}
\centering
\small
\begin{tabular}{lll}
\toprule
\textbf{Member ID} & \textbf{Forecast identity} & \textbf{Used as} \\
\midrule
Energy-1 & Ordinary Least Squares (OLS) & Ensemble member \\
Energy-2 & Ridge & Ensemble member \\
Energy-3 & Bayesian Ridge & Ensemble member \\
Energy-4 & ElasticNet & Ensemble member \\
Energy-5 & Huber Regressor & Ensemble member \\
Energy-6 & Least Angle Regression (LARS) & Ensemble member \\
Energy-7 & Lasso & Ensemble member \\
Energy-8 & LassoLARS & Ensemble member \\
Energy-9 & Linear Support Vector Regression & Ensemble member \\
Energy-10 & Orthogonal Matching Pursuit & Ensemble member \\
\bottomrule
\end{tabular}\label{tab:energy_members}
\end{table}

\subsection{Tropical cyclone ensemble members}\label{app:hurricane_members}

Table \ref{tab:methods} lists the 7 forecast models included as ensemble members and the 2 operational benchmark models and summarizes the predictive methodologies employed.

\begin{table}[h]
\caption{Summary of all tropical cyclone intensity forecasting models used as ensemble members or benchmarks.}
\centering
\small
\resizebox{\textwidth}{!}{%
\begin{tabular}{llll}
\toprule
\textbf{Model ID} & \textbf{Model name} & \textbf{Model type} & \textbf{Used as} \\
\midrule
HUML-(stat, xgb) & Hurricast (tabular data) & Gradient-boosted trees & Ensemble member \\
HUML-(stat/viz, xgb/td) & Hurricast (tabular and vision data) & Feature extraction with tensor decomposition + Gradient-boosted trees & Ensemble member \\
HUML-(stat/viz, xgb/cnn/gru) & Hurricast (tabular and vision data) & Feature extraction with deep learning (CNN, GRU) + Gradient-boosted trees & Ensemble member \\
HUML-(stat/viz, xgb/cnn/transfo) & Hurricast (tabular and vision data) & Feature extraction with deep learning (CNN, Transformers) + Gradient-boosted trees & Ensemble member \\
Decay-SHIPS & Decay Statistical Hurricane Intensity Prediction Scheme & Statistical-dynamical & Ensemble member \\
GFSO & Global Forecast System model & Multi-layer global dynamical & Ensemble member \\
HWRF & Hurricane Weather Research and Forecasting model & Multi-layer regional dynamical & Ensemble member \\
FSSE & Florida State Super Ensemble & Corrected consensus & Benchmark model \\
OFCL & Official National Hurricane Center forecast & Consensus & Benchmark model \\
\bottomrule
\end{tabular}}
\label{tab:methods}
\end{table}

\subsection{Safi coefficient behavior}\label{app:safi_coefs}

The ability of regression coefficients to adapt over time is demonstrated in Figure \ref{fig:evolution}, which plots a subset of these values over a period of two weeks. The rate of change in the coefficients varies across models, with some showing more significant shifts than others. Furthermore, as depicted in Figure \ref{fig:coefs}, the coefficients can take on both positive and negative values. A daily cyclical pattern is also observed, which aligns with the cyclical nature of the forecast errors.

It is worth noting that the predictions of some models may compensate for one another, particularly when they are highly correlated. In such cases, practitioners may want to consider selecting a subset of models before using an adaptive ensemble, and may impose an additional constraint that the sum of their coefficients equals 1, depending on the specific application.

\begin{figure}[h]
\centering
\includegraphics[width=0.95\linewidth]{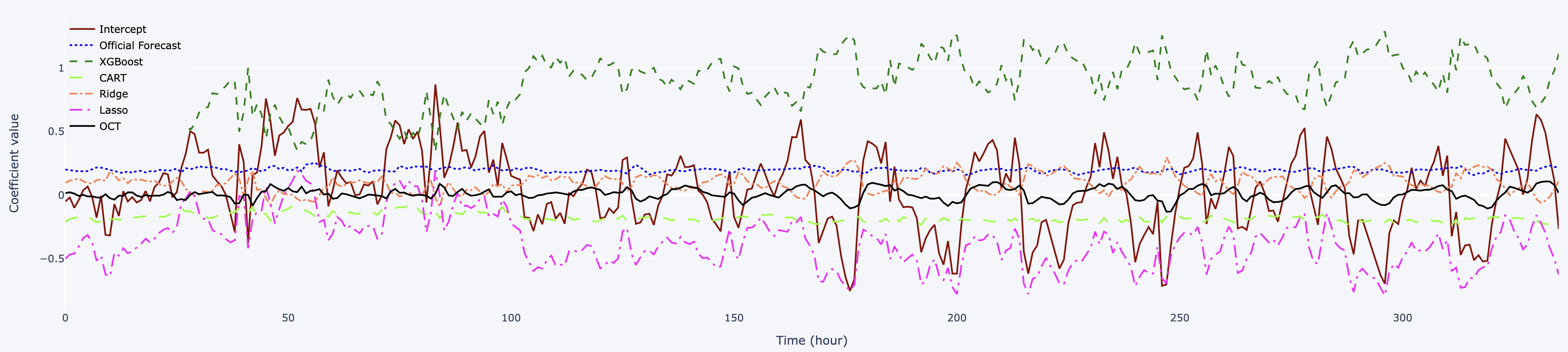}
\caption{Hourly evolution of the fitted coefficients of the ensemble members and of the intercept. In the legend, CART denotes Decision Trees, OCT Optimal Regression Trees, and Official Forecast the Moroccan weather agency forecast (Table~\ref{tab:safi_members}).}
\label{fig:evolution}
\end{figure}

\begin{figure}[h]
\centering
\includegraphics[width=0.95\linewidth]{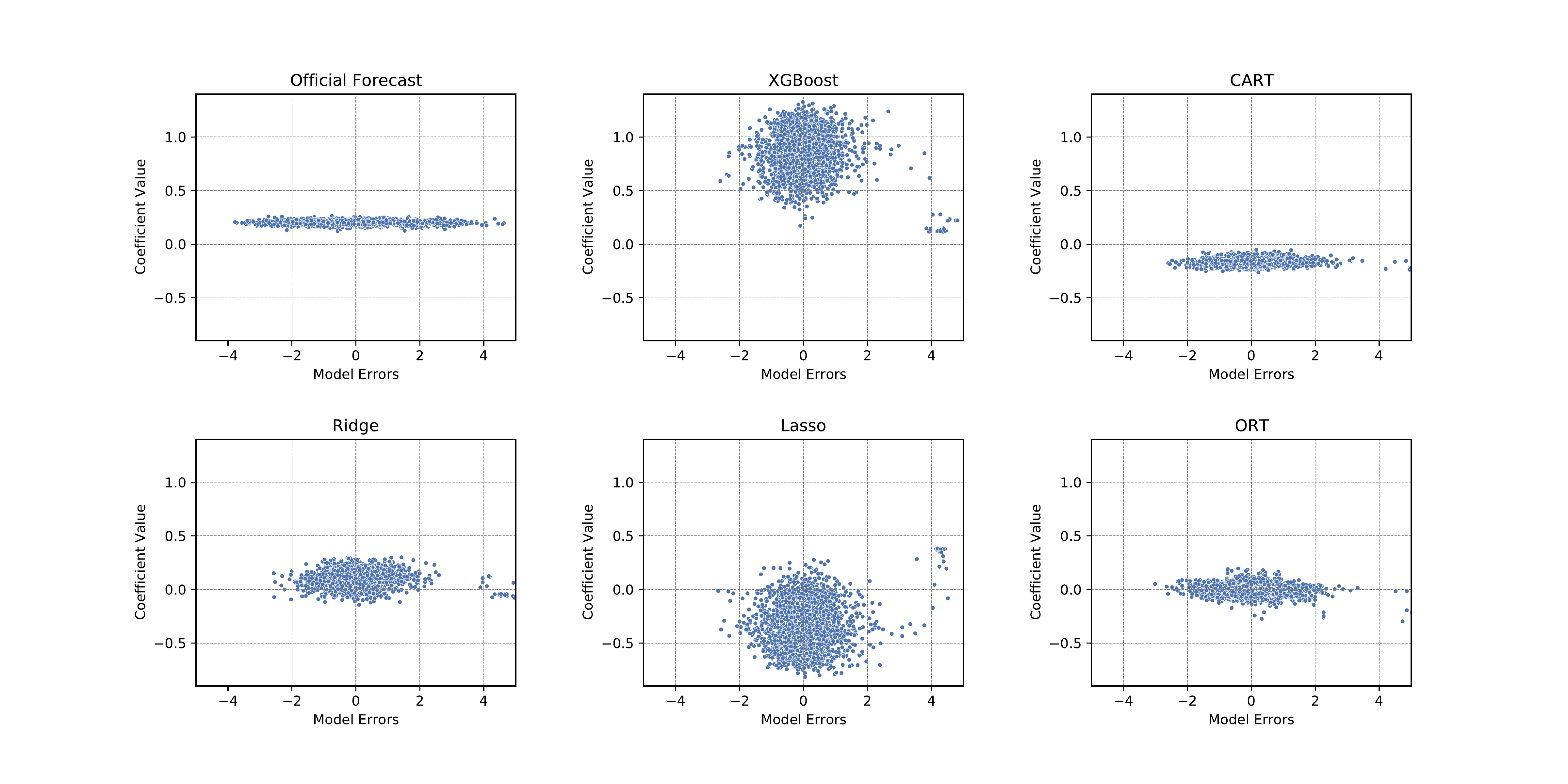}
\caption{Coefficient values of each ensemble member with respect to their errors. CART denotes Decision Trees, ORT Optimal Regression Trees, and Official Forecast the Moroccan weather agency forecast (Table~\ref{tab:safi_members}).}
\label{fig:coefs}
\end{figure}

\section{Broader impacts}\label{app:impacts}

More accurate and robust ensemble forecasts can reduce costly errors in applications such as pollution management, energy planning, and tropical cyclone warnings, where lowering the largest errors matters most. The main risk is over-reliance. Because the fitted error-to-weight map is held fixed during testing, the method can perform poorly after an abrupt shift that is absent from the historical data (Appendix~\ref{app:history}), and operators who trust its weights in that situation could act on poor forecasts. In deployment, we therefore recommend monitoring recent errors and keeping simpler online combiners or official forecasts available as a fallback.

\end{document}